\def\year{2020}\relax
\documentclass[letterpaper]{article} 
\usepackage{aaai20}  
\usepackage{times}  
\usepackage{helvet} 
\usepackage{courier}  
\usepackage[hyphens]{url}  
\usepackage[utf8]{inputenc} 
\usepackage{url}            
\usepackage{booktabs}       
\usepackage{amsfonts}       
\usepackage{nicefrac}       
\usepackage{comment}

\usepackage[T1]{fontenc}
\usepackage{amssymb}
\usepackage{xspace}
\usepackage{thm-restate}
\usepackage{booktabs}
\usepackage{url}
\usepackage{nicefrac}
\usepackage{algorithm,algpseudocode}
\usepackage{caption, subcaption}
\usepackage[usenames,dvipsnames]{xcolor}
\usepackage{mathtools}

\DeclareRobustCommand{\eg}{e.g.,\@\xspace}
\DeclareRobustCommand{\ie}{i.e.,\@\xspace}
\DeclareRobustCommand{\wrt}{w.r.t.\@\xspace}

\DeclareMathOperator{\eig}{eig}

\newcommand{\vtheta}{\boldsymbol{\theta}}
\newcommand{\vmu}{\boldsymbol{\mu}}

\newcommand{\vx}{\boldsymbol{x}}

\newcommand{\Reals}{\mathbb{R}}

\newcommand{\de}{\,\mathrm{d}}

\newcommand{\norm}[2][]{\left\|#2\right\|_{#1}}

\DeclareMathOperator*{\EV}{\mathbb{E}}
\DeclareMathOperator*{\argmax}{arg\,max}

\newcommand{\Dataset}{\mathcal{D}}
\newcommand{\hnabla}{\widehat{\nabla}}

\newcommand{\Aspace}{\mathcal{A}}
\newcommand{\Sspace}{\mathcal{S}}

\newcommand{\Init}{\mu}

\newcommand{\Rmax}{R_{\max}}

\newcommand{\sm}{\psi}

\newcommand{\BlackBox}{\rule{1.5ex}{1.5ex}}  
\newenvironment{proof}{\par\noindent{\bf Proof\ }}{\hfill\BlackBox\\[2mm]}
 
\newtheorem{theorem}{Theorem}
\newtheorem{lemma}[theorem]{Lemma} 
\newtheorem{proposition}[theorem]{Proposition}

\newcommand{\citet}[1]{citeauthor{#1}\shortcite{#1}}
\newcommand{\citep}{\cite}

\title{Risk-Averse Trust Region Optimization \\for Reward-Volatility Reduction}
\author{
Lorenzo Bisi$^{1,2}$,
Luca Sabbioni$^{1,2}$,
Edoardo Vittori$^3$,
Matteo Papini$^1$,
Marcello Restelli$^{1,2}$\\
$^1$Politecnico di Milano, Milan, Italy\\
$^2$ Institute for Scientific Interchange Foundation, Turin, Italy \\
$^3$Banca IMI, Milan, Italy\\
\{lorenzo.bisi, luca.sabbioni, matteo.papini, marcello.restelli\}@polimi.it,
edoardo.vittori@bancaimi.it
}

\begin{document}

\maketitle

\begin{abstract}
In real-world decision-making problems, for instance in the fields of finance, robotics or autonomous driving, keeping uncertainty under control is as important as maximizing expected returns.
Risk aversion has been addressed in the reinforcement learning literature through risk measures related to the variance of returns.
However, in many cases, the risk is measured not only on a long-term perspective, but also on the step-wise rewards (e.g., in trading, to ensure the stability of the investment bank, it is essential to monitor the risk of portfolio positions on a daily basis).
In this paper, we define a novel measure of risk, which we call reward volatility, consisting of the variance of the rewards under the state-occupancy measure. We show that the reward volatility bounds the return variance so that reducing the former also constrains the latter.
We derive a policy gradient theorem with a new objective function that exploits the mean-volatility relationship, and develop an actor-only algorithm.
Furthermore, thanks to the linearity of the Bellman equations defined under the new objective function, it is possible to adapt the well-known policy gradient algorithms with monotonic improvement guarantees such as TRPO in a risk-averse manner. Finally, we test the proposed approach in two simulated financial environments.
\end{abstract}
\section{Introduction}
Reinforcement Learning (RL)~\citep{Suttonbook:1998} methods have recently grown in popularity in many types of applications. Powerful policy search~\cite{deisenroth2013survey} algorithms, such as TRPO \citep{schulman2015trust} and PPO \citep{schulman2017proximal}, give very exciting results~\cite{OpenAI_dota,Heess2017Emergence} in terms of efficiently maximizing the expected value of the cumulative discounted rewards (referred to as expected return). These types of algorithms are proving to be very effective in many sectors, but they are leaving behind problems in which maximizing the return is not the only goal and risk aversion becomes an important objective too. Risk-averse reinforcement learning is not a new theme: a utility based approach has been introduced in  \citep{shen_risk-averse_2014} and \citep{moldovan_risk_2012} where the value function becomes the expected value of a utility function of the reward. A category of objective functions called \emph{coherent risk functions}, characterized by convexity, monotonicity, translation invariance and positive homogeneity, has been defined and studied in \citep{tamar_sequential_2017}. These include known risk functions such as CVaR (Conditional Value at Risk) and mean-semideviation.
Another category of risk-averse objective functions are those which include the variance of the returns (referred to as return variance throughout the paper), which is then combined with the standard return in a mean-variance \citep{tamar_variance_2013,pra-ghav} or a Sharpe ratio~\citep{moody2001learning} fashion. \\
In certain domains, return variance and CVaR are not suitable to correctly capture risk. In finance, for instance, keeping a low return variance could be appropriate in the case of long term investments, where performance can be measured on the Profit and Loss (P\&L) made at the end of the year. However, in any other type of investment, interim results are evaluated frequently, thus keeping a low-varying \emph{daily} P\&L becomes crucial. \\ 
This paper analyzes, for the first time, the variance of the reward at each time step \wrt state visitation probabilities. We call this quantity \emph{reward volatility}.
Intuitively, the return variance measures the variation of accumulated reward among trajectories, while reward volatility is concerned with the variation of single-step rewards among visited states. 
Reward volatility is used to define a new risk-averse performance objective which trades off the maximization of expected return with the minimization of short-term risk (called \textit{mean-volatility}).\\
The main contribution of this paper is the derivation of a policy gradient theorem for the new risk-averse objective, based on a novel risk-averse Bellman equation. 
These theoretical results are made possible by the simplified tractability of reward volatility compared to return variance, as the former lacks the problematic inter-temporal correlation terms of the latter.
However, we also show that reward volatility upper bounds the return variance (albeit for a normalization term). This is an interesting result, indicating that it is possible to use the analytic results we derived for the reward volatility to keep under control the return variance. \\
If correctly optimized, the mean-volatility objective allows to limit the \textit{inherent risk} due to the stochastic nature of the environment. However, the imperfect knowledge of the model parameters, and the consequent imprecise optimization process, is another relevant source of risk, known as \textit{model risk}. This is especially important when the optimization is performed on-line, as may happen for an autonomous, adaptive trading system. To avoid any kind of performance oscillation, the intermediate solutions implemented by the learning algorithm must guarantee continuing improvement. The TRPO algorithm provides this kind of guarantees (at least in its ideal formulation) for the risk-neutral objective.
The second contribution of our paper is the derivation of the Trust Region Volatility Optimization (TRVO) algorithm, a TRPO-style algorithm for the new mean-volatility objective. 
After some background on policy gradients (Section~\ref{sec:pre}), the volatility measure is introduced in Section~\ref{sec:risk_measure} and compared to the return variance. The Policy Gradient Theorem for the mean-volatility objective is provided in Section~\ref{sec:grad}. In Section~\ref{sec:estimator}, we introduce an estimator for the gradient which is based on the sample trajectories obtained from direct interaction with the environment, which yields a practical risk-averse policy gradient algorithm (VOLA-PG). The TRVO algorithm is introduced in Section~\ref{sec:TRVO}.
Finally, in Section~\ref{sec:experiments}, we test our algorithms on two simulated financial environments; in the former the agent has to balance investing in liquid and non-liquid assets with simulated dynamics while in the latter the agent has to learn trading on a real asset using historical data.
\section{Preliminaries}\label{sec:pre}
A discrete-time Markov Decision Process (MDP) is defined as a tuple $\langle\Sspace,\Aspace, \mathcal{P}, \mathcal{R}, \gamma, \mu\rangle$, where $\Sspace$ is the (continuous) state space, $\Aspace$ the (continuous) action space, $\mathcal{P}(\cdot|s,a)$ is a Markovian transition model that assigns to each state-action pair $(s,a)$ the probability of reaching the next state $s'$, $\mathcal{R}$ is a bounded reward function, i.e. $\sup_{s\in\Sspace, a\in\Aspace}|\mathcal{R}(s,a)|\leq\Rmax$ , $\gamma\in[0,1)$ is the discount factor, and $\mu$ is the distribution of the initial state. The policy of an agent is characterized by $\pi(\cdot|s)$, which assigns to each state $s$ the density distribution over the action space $\Aspace$.\\
We consider infinite-horizon problems in which future rewards are exponentially discounted with~$\gamma$.  Following a trajectory $\tau \coloneqq (s_0, a_0, s_1, a_1, s_2, a_2, ...$), let the returns be defined as the discounted cumulative reward: $    G_\tau = \sum_{t=0}^\infty \gamma^t \mathcal{R}(s_t,a_t).$
For each state $s$ and action $a$, the action-value function is defined as:
\begin{equation}
    Q_\pi(s,a) \coloneqq \EV_{\substack{s_{t+1}\sim \mathcal{P}(\cdot|s_{t},a_{t})\\a_{t+1}\sim\pi(\cdot|s_{t+1})}}\left[\sum_{t=0}^\infty \gamma^t \mathcal{R}(s_t,a_t)|s_0 = s, a_0 = a\right], 
    \label{eq:Q_fun}
\end{equation}
which can be recursively defined by the following Bellman equation:
\begin{equation*}
        Q_\pi(s,a) = \mathcal{R}(s,a) + \gamma \EV_{\substack{s'\sim \mathcal{P}(\cdot|s,a)\\a'\sim\pi(\cdot|s')}}\big[Q_\pi(s',a')\big].
\end{equation*}
For each state $s$, we define the state-value function of the stationary policy $\pi(\cdot|s)$ as:
\begin{equation}
\begin{aligned}
    V_\pi(s) &\coloneqq \EV_{\substack{a_t\sim\pi(\cdot|s_t)\\s_{t+1}\sim \mathcal{P}(\cdot|s_{t},a_{t})}} \left[\sum_{t=0}^\infty \gamma^t \mathcal{R}(s_t,a_t)|s_0 = s\right] \\& = \EV_{a\sim\pi(\cdot|s)}\big[Q(s,a)\big]. 
    \label{eq:V_fun}
    \end{aligned}
\end{equation}
It is useful to introduce the (discounted) state-occupancy measure induced by $\pi$:
\begin{align*}
	d_{\Init,\pi}(s) &\coloneqq (1-\gamma)\int_{S}\mu(s_0)\sum_{t=0}^{\infty}\gamma^tp_{\pi}(s_0 \xrightarrow{t} s)\de s_0,\label{def:dmupi}
\end{align*}
where $p_{\pi}(s_0 \xrightarrow{t} s)$ is the probability of reaching state $s$ in $t$ steps from $s_0$ following policy $\pi$ (see Appendix~\ref{app:proofs}).

\subsection{Actor-only policy gradient}
\label{ssec:policy_grad}
From the previous subsection, it is possible to define the \emph{normalized}\footnote{In our notation, the expected return (as commonly defined in the RL literature) is ${J_{\pi}}\big/{(1-\gamma)}.$} expected return $J_\pi$ using two distinct formulations, one based on transition probabilities and the other on state occupancy $d_{\mu,\pi}$:
\begin{align*}
    J_\pi & \coloneqq (1-\gamma)\EV_{\substack{s_0\sim \mu \\a_t\sim\pi(\cdot|s_t)\\s_{t+1}\sim \mathcal{P}(\cdot|s_{t},a_{t})}}\left[\sum_{t=0}^\infty \gamma^t \mathcal{R}(s_t,a_t)\right] 
    \\ &= \EV_{\substack{s\sim d_{\mu,\pi}\\a\sim\pi(\cdot|s)}}\left[ \mathcal{R}(s,a)\right].
\end{align*}    

For the rest of the paper, we consider parametric policies, where the policy $\pi_{\vtheta}$ is parametrized by a vector $\vtheta\in \Theta \subseteq \mathbb{R}^m$. In this case, the goal is to find the optimal parametric policy maximizing the performance, i.e. $\vtheta^* = \arg\max_{\vtheta\in\Theta}J_\pi$.\footnote{For the sake of brevity, when a variable depends on the policy $\pi_{\vtheta}$, in subscripts only $\pi$ is shown, omitting the dependency on the parameters $\vtheta$.}
The Policy Gradient Theorem (PGT)~\citep{Suttonbook:1998} states that, for a given policy $\pi_{\vtheta},\  \vtheta\in\Theta$:
\begin{equation}\label{eq:PGT}
    \nabla_{\vtheta} J_\pi = \EV_{\substack{s\sim d_{\mu,\pi}\\a\sim\pi_{\vtheta}(\cdot|s)}}\bigg[\nabla_{\vtheta}\log\pi_{\vtheta}(a|s)Q_\pi(s,a)\bigg].
\end{equation}
In the following, we omit the gradient subscript whenever clear from the context.
\section{Risk Measures}\label{sec:risk_measure}
This section introduces the concept of reward volatility, comparing it with the more common return variance.
The latter, denoted with $\sigma_{\pi}^2$, is defined as:
\begin{equation}
     \sigma_\pi^2  \coloneqq \EV_{\substack{s_0 \sim \mu \\a_t\sim\pi_{\vtheta}(\cdot|s_t)\\ s_{t+1}\sim \mathcal{P}(\cdot|s_t,a_t)}}\left[\left(\sum_{t=0}^\infty \gamma^t \mathcal{R}(s_t,a_t)-\frac{J_\pi}{1-\gamma}\right)^2\right].
     \label{eq:path_variance}
\end{equation}
In our case, it is useful to define \textit{reward volatility} $\nu^2_\pi$ in terms of the distribution $d_{\mu,\pi}$. As it is not possible to define the return variance in the same way, we also rewrite reward volatility as an expected sum over trajectories:\footnote{In finance, the term ``volatility'' refers to a generic measure of variation, often defined as a standard deviation. In this paper, volatility is defined as a variance.}
\begin{align}
    \nu^2_\pi & \coloneqq 
    \EV_{\substack{s\sim d_{\mu,\pi}\\a\sim\pi_{\vtheta}(\cdot|s)}}\left[\left(\mathcal{R}(s,a)-J_\pi\right)^2\right] \\ 
    & = (1-\gamma)\EV_{\substack{s_0 \sim \mu \\a_t\sim\pi_{\vtheta}(\cdot|s_t)\\ s_{t+1}\sim \mathcal{P}(\cdot|s_t,a_t)}}\left[\sum_{t=0}^\infty \gamma^t \left(\mathcal{R}(s_t,a_t)-J_\pi\right)^2\right].
    \label{eq:sigma} 
\end{align}
Once we have set a mean-variance parameter $\lambda$, the \textit{performance} or objective function related to the policy $\pi$ can be defined as:
\begin{equation}\label{eq:eta}
    \eta_\pi \coloneqq J_\pi - \lambda \nu^2_\pi,
\end{equation}
called \textit{mean-volatility} hereafter, where $\lambda\geq 0$ allows to trade-off expected return maximization with risk minimization. Similarly, the mean-variance objective is $J_{\pi}\big/(1-\gamma) - \lambda\sigma_{\pi}^2$.
An important result on the relationship between the two variance measures is the following:
\begin{restatable}[]{lemma}{variance}\label{thm:variance_ineq}
Consider the return variance $\sigma_\pi^2$ defined in Equation~\eqref{eq:path_variance} and the reward volatility $\nu^2_\pi$ defined in Equation~\eqref{eq:sigma}. The following inequality holds:
\begin{equation*}
 \sigma_\pi^2 \leq \frac{\nu^2_\pi}{(1-\gamma)^2} .
\end{equation*}
\end{restatable}
The proofs for this and other formal statements can be found in Appendix~\ref{app:proofs}. It is important to notice that the factor $(1-\gamma)^2$ simply comes from the fact that the return variance is not normalized, unlike the reward volatility (intuitively, volatility measures risk on a shorter time scale). What is lost in the reward volatility compared to the return variance are the inter-temporal correlations between the rewards.
However, Lemma~\ref{thm:variance_ineq} shows that the minimization of the reward volatility yields a low return variance. The opposite is clearly not true: a counterexample can be a stock price, having the same value at the beginning and at the end of a day, but making complex movements in-between.
\begin{figure}[t]
    \begin{center}
        \includegraphics[width=\columnwidth]{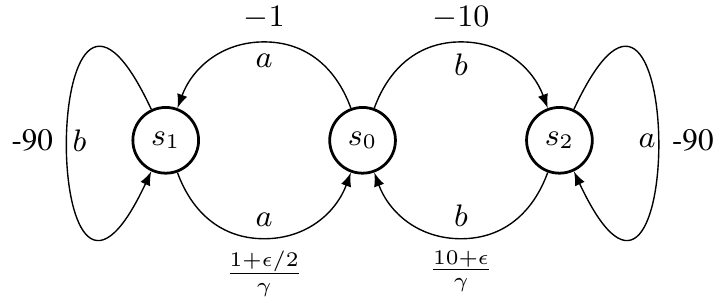}
        \caption{A deterministic MDP. The available actions are $a$ and $b$, $s_0$ is the initial state and rewards are reported on the arrows.}
        \label{fig:simpleMDP}
    \end{center}
\end{figure}
To better understand the difference between the two types of variance, consider the deterministic MDP in Figure \ref{fig:simpleMDP}. First assume $\epsilon = 0$. Every optimal policy (thus avoiding the $-90$ rewards) yields an expected return $J_\pi = 0$. The reward volatility of a deterministic policy 
that repeats the action $a$ is $\nu^2_{a} = \nicefrac{1}{\gamma}$ while the reward volatility of 
repeating the action $b$ is $\nu^2_{b} = \nicefrac{100}{\gamma}$. If we were minimizing the reward volatility, we would prefer the first policy, while we would be indifferent between the two policies based on the return variance ($\sigma_\pi^2$ is 0 in both cases). Now let us complicate the example, setting $\epsilon \in (0,1)$. The returns are now $J_{b} = \frac{\epsilon}{1+\gamma} > \frac{\epsilon/2}{1+\gamma} = J_{a}$. As a consequence, a mean-variance objective would always choose action $b$, since the return variance is still $0$, while the mean-volatility objective may choose the other path, depending on the value of the risk-aversion parameter $\lambda$. This simple example shows how the mean-variance objective can be insensitive to short-term risk (the $-10$ reward), even when the gain in expected return is very small in comparison ($\epsilon\simeq 0$). Instead, the mean-volatility objective correctly captures this kind of trade-off.
\section{Risk-Averse Policy Gradient}\label{sec:grad}
In this section, we derive a policy gradient theorem for the reward volatility $\nu_\pi^2$, and we propose an unbiased gradient estimator.
This will allow us to solve the optimization problem $\max_{\vtheta\in\Theta}\eta_{\vtheta}$ via stochastic gradient ascent.
We introduce a volatility equivalent of the action-value function $Q_\pi$ (Equation~\eqref{eq:Q_fun}), which is the volatility observed by starting from state $s$ and taking action $a$ thereafter:
\begin{equation}\label{eq:X_fun}
    X_\pi(s,a) \coloneqq \EV_{\substack{s_{t+1}\sim P(\cdot|s_t,a_t)\\a_{t+1}\sim\pi_{\vtheta}(\cdot|s_{t+1}.)}}\left[\sum_{t=0}^\infty \gamma^t (\mathcal{R}(s_t,a_t) - J_\pi)^2|s, a\right],
\end{equation}
called \emph{action-volatility} function.
 Like the $Q$ function, $X$ can be written recursively by means of a Bellman equation:
\begin{equation}
    X_\pi(s,a) = \big(R(s,a)-J_\pi\big)^2 + \gamma \EV_{\substack{s'\sim P(\cdot|s,a)\\a'\sim\pi_{\vtheta}(\cdot|s')}}\big[X_\pi(s',a')\big].
    \label{eq:bellmanV}
\end{equation}
The \emph{state-volatility} function, which is the equivalent of the $V$ function (Equation~\eqref{eq:V_fun}), can then be defined as:
\begin{align*}
     W_\pi(s) &\coloneqq  \EV_{\substack{a_t\sim\pi(\cdot|s_t)\\s_{t+1}\sim \mathcal{P}(\cdot|s_{t},a_{t})}}\left[\sum_{t=0}^\infty \gamma^t \left(\mathcal{R}(s_t,a_t)-J_\pi\right)^2|s\right] \\
     &=\EV_{a\sim\pi_{\vtheta}(\cdot|s)}\big[X_\pi(s,a)\big].
\end{align*}
This allows to derive a Policy Gradient Theorem (PGT) for $\nu_{\pi}^2$, as done in (Sutton et al. 2000) for the expected return:
\begin{restatable}[Reward Volatility PGT]{theorem}{gradient}\label{thm:variance_grad}
Using the definitions of action-variance and state-variance function, the variance term $\nu^2_\pi$ can be rewritten as:
\begin{equation}\label{eq:sigmatilde}
    \nu_\pi^2 = (1-\gamma)\int_\Sspace \mu(s)W_\pi(s)ds.
\end{equation}
Moreover, for a given policy $\pi_{\vtheta},\  \vtheta\in\Theta$:
\begin{equation*}
    \nabla\nu^2_\pi = \EV_{\substack{s\sim d_{\mu,\pi}\\a\sim\pi_{\vtheta}(\cdot|s)}}\bigg[\nabla\log\pi_{\vtheta}(a|s)X_\pi(s,a)\bigg].
\end{equation*}
\end{restatable}
The existence of a linear Bellman equation (\ref{eq:bellmanV}) and the consequent policy gradient theorem are two non-trivial theoretical advantages of the new reward volatility approach with respect to the return variance. 
With a simple extension it is possible to obtain the policy gradient theorem for the mean-volatility objective defined in equation (\ref{eq:eta}). The action value and state value functions are obtained by combining the action value functions of the expected return (\ref{eq:Q_fun}) and of the volatility (\ref{eq:X_fun}); the same holds for the state value functions:
\begin{align*}
    Q^\lambda_\pi(s,a) &\coloneqq Q_\pi(s,a)-\lambda X_\pi(s,a) \\
    V^\lambda_\pi(s) &= V_\pi(s) -\lambda W_\pi(s)
\end{align*}
The policy gradient theorem thus states:
\begin{equation*}
    \nabla\eta_\pi = \EV_{\substack{s\sim d_{\mu,\pi}\\a\sim\pi_{\vtheta}(\cdot|s)}}\bigg[\nabla\log\pi_{\vtheta}(a|s)Q^\lambda_\pi(s,a)\bigg].
\end{equation*}
In the following part, we use these results to design a VOLatility-Averse Policy Gradient algorithm (VOLA-PG).

\subsection{Estimating the risk-averse policy gradient}\label{sec:estimator}
To design a practical actor-only policy gradient algorithm, the action-value function $Q_\pi$ needs to be estimated as in~\citep{Suttonbook:1998,peters2008reinforcement}.
Similarly, we need an estimator for the action-variance function $X_\pi$. In this approximate framework, we consider to collect $N$ finite trajectories $s_0^i, a_0^i, ... , s^i_{T-1}, a^i_{T-1}, \  i = 0,\dots,N-1$ per each policy update. An unbiased estimator of $J_\pi$ can be defined as:
\begin{equation}
   \hat{J} =  \frac{1-\gamma}{1-\gamma^{T}}\frac{1}{N} \sum_{i=0}^{N-1} \sum_{t=0}^{T-1} \gamma^{t}\mathcal{R}^i_{t}
    \label{eq:J_hat},
\end{equation}
where rewards are denoted as $\mathcal{R}_t^i = \mathcal{R}(s_t^i, a_t^i)$.
This can be used to compute an estimator for the action-volatility function:
\begin{restatable}[]{lemma}{xestimator}\label{lemma:X_estimate}
Let $\widehat{X}$ be the following estimator for the action-volatility function:
\begin{equation}
    \widehat{X} = \frac{1-\gamma}{1-\gamma^{T}}\frac{1}{N} \sum_{i=0}^{N-1} \sum_{t=0}^{T-1} \gamma^t \left[(\mathcal{R}_t^i-\hat{J}_1)(\mathcal{R}_t^i-\hat{J}_2)\right].
    \label{eq:X_hat}
\end{equation}
where $\hat{J}_1$ and $\hat{J}_2$, defined as in Equation~\eqref{eq:J_hat}, are taken from two different sets of trajectories $ \mathcal{D}_1$ and $\mathcal{D}_2$, and a third set of samples $\mathcal{D}_3$ is used for the rewards $\mathcal{R}_t^i$ in Equation~\eqref{eq:X_hat}.\\
Then, $\widehat{X}$ is unbiased.
\end{restatable}


Note that, in order to obtain an unbiased estimator for $X$, a \textit{triple sampling} procedure is needed. This may result in being very restrictive. However, by adopting single sampling instead, the bias introduced is equivalent to the variance of $\widehat{J}$, so the estimator is still consistent.
This result can be used to build a consistent estimator for the gradient $\nabla\eta_\pi$:\footnote{Obtained by adopting single sampling.}
\begin{equation}\label{eq:approx_gradient}
\begin{aligned}
    \widehat{\nabla}_N
    \eta_\pi=& \frac{1}{N}\sum_{i=0}^{N-1}\sum_{t=0}^{T-1}\gamma^t\bigg(\sum_{t'=t}^{T-1}\gamma^{t' - t}\big[R_{t'}^i-\lambda\frac{1-\gamma}{1-\gamma^T}\\ &\ (R_{t'}^i-\widehat{J})^2 \big]\bigg)\nabla\log\pi_{\vtheta}(a_t^i|s_t^i).
    \end{aligned}
\end{equation}
This is just a PGT \citep{sutton2000policy} estimator with $R_{t'}^i-\lambda\frac{1-\gamma}{1-\gamma^T}(R_{t'}^i-\widehat{J})^2$ in place of the simple reward. As shown in \citep{peters2008reinforcement}, this can be turned into a GPOMDP estimator, for which variance-minimizing baselines can be easily computed.
Pseudocode for the resulting actor-only policy gradient method is reported in Algorithm~\ref{alg:pg}.


\begin{algorithm}[t]
\caption{Volatility-Averse Policy Gradient (VOLA-PG)}
\label{alg:pg}
	\begin{algorithmic}

		\State \textbf{Input:} initial policy parameter $\vtheta_0$, batch size $N$, number of iterations $K$, learning-rate $\alpha$.
		\For{$k=0,\dots,K-1$}
			\State Collect $N$ trajectories with $\vtheta_k$ to obtain dataset $\Dataset_{N}$
			\State Compute estimates $\widehat{J}$ as in Equation~\eqref{eq:J_hat}
			\State Estimate gradient $\widehat{\nabla}_N\eta_{\vtheta_k}$ as in Equation~\eqref{eq:approx_gradient}
			\State Update policy parameters as $\vtheta_{k+1} \gets \vtheta_k + \alpha\hnabla_N \eta_{\vtheta_k}$
		\EndFor
	\end{algorithmic}
\end{algorithm}
\section{Trust Region Volatility Optimization}\label{sec:TRVO}
In this section, we go beyond the standard policy gradient theorem and show it is possible to guarantee a monotonic improvement of the mean-volatility performance measure~\eqref{eq:eta} at each policy update. Safe (in the sense of non-pejorative) updates are of fundamental importance when learning online on a real system, such as when controlling a robot or when trading in the financial markets; but also give interesting results during offline training, guaranteeing an automatic tuning of the optimal learning rate.
While the mean-volatility objective ensures a risk averse \emph{behavior} of the policy, the safe update ensures a risk averse \emph{update} of the paramenters of the policy. Thus, if we care about the agent's performance within the learning process, we must emphasize the importance of the step sizes at each update of the parameters. Adapting the approach in \citep{schulman2015trust} to our mean-volatility objective, we show it is possible to obtain a learning rate that guarantees that the performance of the updated policy is bounded with respect to the previous policy. The safe update is based on the advantage function, defined as the difference between the action value and state value function, because of the linearity of the Bellman equations, we can extend the definition to obtain the mean-volatility advantage function:
\begin{equation}\label{eq:B-dvantage}
\begin{aligned}
A^\lambda_\pi(s,a) =&  Q^\lambda_\pi(s,a) - V^\lambda_\pi(s) \\
                =& \EV_{s'\sim P(\cdot|s,a)}\big[\mathcal{R}(s,a) - \lambda (\mathcal{R}(s,a)-J_\pi)^2 \\
                &\qquad  \qquad +\gamma V_\pi^\lambda(s')-V_\pi^\lambda(s)\big]
\end{aligned}
\end{equation}
Furthermore, with the mean-volatility objective all the theoretical results leading to the \textit{TRPO} algorithm hold. The main ones are in this section, but the full derivation can be found in Appendix~\ref{app:proofs}. Lemma \ref{lemma:ADV_kakade} is a $\lambda$-extension of Lemma 6.1 in \citep{kakade2002approximately}:\footnote{The only slight difference is in the normalization term, consequent to the choice of different definitions.}
\begin{restatable}[Performance Difference Lemma]{lemma}{ADVlambda}\label{lemma:ADV_kakade}
The difference of the performances between to two policies $\pi$ and $\widetilde{\pi}$ can be expressed in terms of expected advantage:
\begin{equation}\label{eq:delta_eta}
\begin{aligned}
\eta_{\widetilde{\pi}} = &\eta_\pi + \int_\Sspace d_{\mu,\widetilde{\pi}}(s)\int_\Aspace\widetilde{\pi}(a|s)A_\pi^\lambda(s,a)\de a \de s +\\
+&\lambda(1-\gamma)^2\bigg[\int_\Sspace d_{\mu, \widetilde{\pi}}(s)\int_\Aspace\widetilde{\pi}(a|s)A_\pi(s,a)\de a \de s\bigg]^2.\end{aligned}\end{equation}\end{restatable}
This result is very interesting, since the last term adds a gain related to the square of the difference of the expected returns of the policies; therefore there is a bonus if the expected return of the second policy is either higher or lower than the first one. However, this is a difficult term to consider, and we can bound the performance difference by neglecting it; in this case the bound becomes:
\begin{equation}\label{eq:delta_eta_approx}
    \eta_{\widetilde{\pi}} \ge \eta_\pi + \int_\Sspace d_{\mu,\widetilde{\pi}}(s)\int_\Aspace\widetilde{\pi}(a|s)A_\pi^\lambda(s,a)\de a \de s
\end{equation}
This result is the one that could be obtained considering a transformation of the MDP with a new, policy-dependent reward
$R_\pi^\lambda(s_t, a_t) = R(s_t, a_t)-\lambda(R(s_t, a_t) -J_\pi)^2$. In any case, the reader must be careful not to underestimate the choice to adopt a reward of this kind, since, in general, for policy-dependent rewards the PGT in Equation~\eqref{eq:PGT} and the performance bound in Equation ~\eqref{eq:delta_eta_approx} do not hold without proper assumptions.

Following the approach proposed in ~\citep{schulman2015trust}, it is possible to adopt an approximation $L_\pi^\lambda(\widetilde{\pi})$ of the surrogate function, which provides monotonic improvement guarantees by considering the KL divergence between the policies:
\begin{restatable}[]{theorem}{TRVO}\label{thm:TRVO}
Consider the following approximation of $\eta_{\widetilde{\pi}}$, replacing the state-occupancy density of the old policy $d_{\mu, \pi}$:
\begin{equation}\label{eq:L_lambda}
L_\pi^\lambda(\widetilde{\pi}) \coloneqq \eta_\pi + \int_\Sspace d_{\mu, \pi}(s)\int_\Aspace \widetilde{\pi}(a|s) A_\pi^\lambda(s,a)\de a \de s;
\end{equation}
Let \begin{align*}
\alpha &= D^{max}_{KL}(\pi,\widetilde{\pi}) = \max_{s} D_{KL}(\pi(\cdot|s), \widetilde{\pi}(\cdot|s))\\ \epsilon& = \max_s |\EV_{a\sim\widetilde{\pi}(\cdot|s)}\big[ A^\lambda_\pi(s,a)\big]|\end{align*}
 Then, the performance of $\widetilde{\pi}$ can be bounded as follows:\footnote{Comparing this bound to the results shown in ~\citep{schulman2015trust}, the denominator term is not squared due to return normalization.}
\begin{equation}\label{eq:TRPOVOLA-bound}
\begin{aligned}
    \eta_{\widetilde{\pi}} \ge L^\lambda_\pi(\widetilde{\pi}) - \frac{2\epsilon\gamma}{1-\gamma} D^{max}_{KL}(\pi, \widetilde{\pi}).
\end{aligned}
\end{equation}
\end{restatable}
As a consequence, we can devise a trust-region variant of VOLA-PG, called TRVO (Trust Region Volatility Optimization) as outlined in Algorithm \ref{alg:trvo}. Note that this would not have been possible without the linear Bellman equations, which are not available for more common risk measures. For the practical implementation, we follow~\citep{schulman2015trust}.

\begin{algorithm}[t]
\caption{Trust Region Volatility Optimization (TRVO)}
\label{alg:trvo}
	\begin{algorithmic}
		\State \textbf{Input:} initial policy parameter $\vtheta_0$, batch size $N$, number of iterations $K$, discount factor $\gamma$.
		\For{$k=0,\dots,K-1$}
			\State Collect $N$ trajectories with $\vtheta_k$ to obtain dataset $\Dataset_{N}$
			\State Compute estimates $\widehat{J}$ as in Equation~\eqref{eq:J_hat}
			\State Estimate advantage values $A^\lambda_{\theta_k}(s,a)$ 
			\State Solve the constrained optimization problem 
			\begin{equation*}
			\begin{multlined}
			    \vtheta_{k+1} = \argmax_{\vtheta\in\Theta}\big[ L^\lambda_{k}(\vtheta) - \frac{2\epsilon\gamma}{1-\gamma} D_{KL}^{max}(\pi_{\vtheta_{k}}, \pi_{\vtheta})\big]\\
			    \text{where  }  \epsilon = \max_s\max_a|A^\lambda_{\vtheta_{k}}(s,a)|\\
			     L^\lambda_{k}(\vtheta) = \eta_{\vtheta_{k}} +  \pi_{\vtheta_{k}}(a|s)\EV_{\substack{s\sim d_{\mu, \pi_{\vtheta_{k}}}\\a\sim\pi_{\vtheta_{k}}(\cdot|s)}}A_{\vtheta_{k}}^\lambda(s,a)
			    \end{multlined}
		    \end{equation*}
		\EndFor
	\end{algorithmic}
\end{algorithm}

\section{Related Works}\label{sec:related}
\begin{figure*}[t]
    \begin{minipage}{0.48\textwidth}
          \begin{figure}[H]
      \includegraphics[width=0.9\columnwidth]{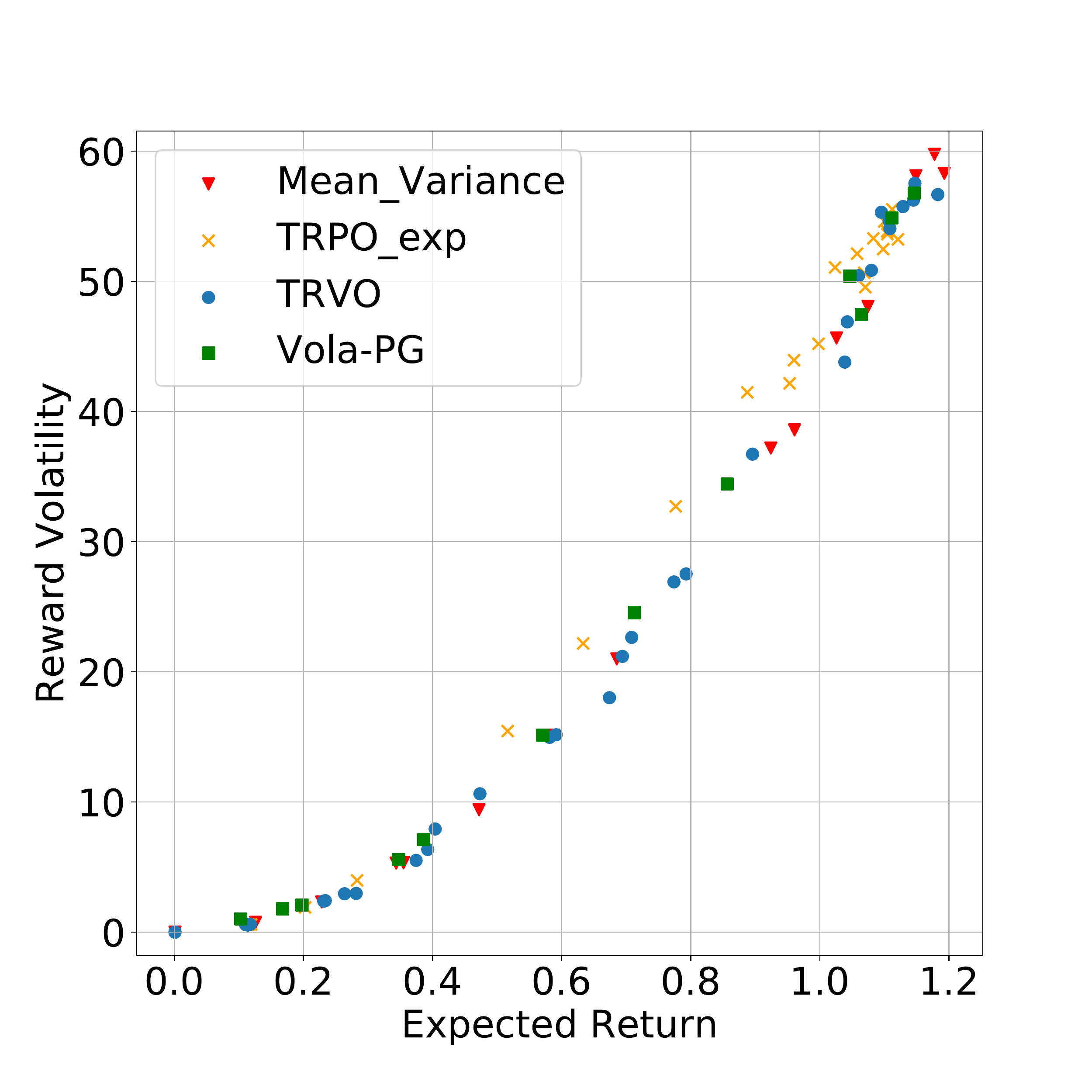}

  \end{figure}

    \end{minipage}
    \begin{minipage}{0.48\textwidth}
      \begin{figure}[H]
      \includegraphics[width=0.9\columnwidth]{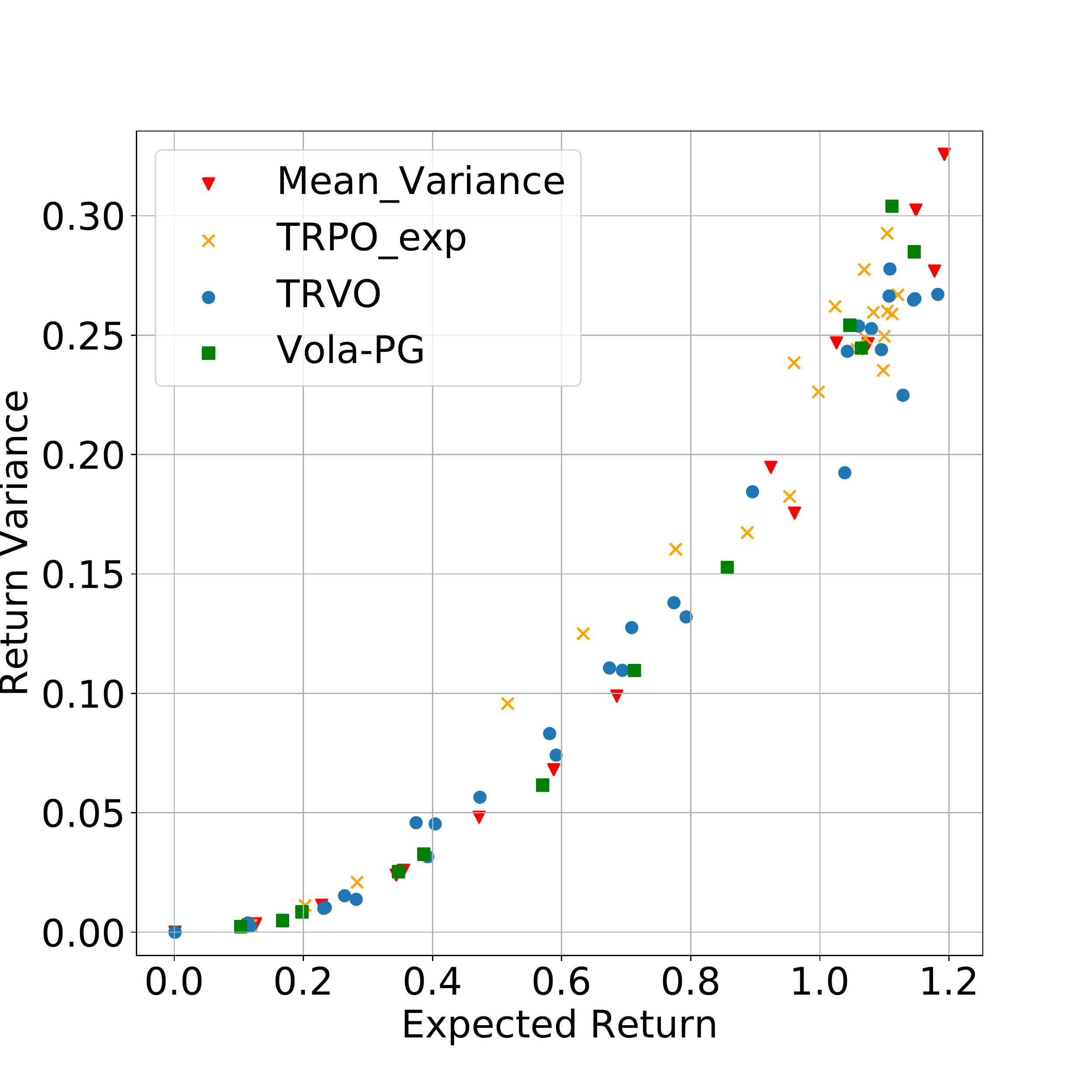}
      
      \label{fig:pareto_tamar_var}
        \end{figure}
    \end{minipage}
      \caption{Portfolio Management: comparison of the expected return, reward volatility and return variance obtained with four different learning algorithms: TRVO, TRPO-exp, Vola-PG and Mean-Variance.}
\label{fig:pareto_tamar}
\end{figure*}
Two streams of RL literature have been merged together in this paper (for the first time, to the best of our knowledge): risk-averse objective functions and safe policy updates. As stated in~\citep{Tamar_variance,garcia15a}, these two themes are related as they both reduce risk.
The first is the reduction of inherent risk, which is generated by the stochastic nature of the environment, while the second is the reduction of model risk, which is related to the imperfect knowledge of the model parameters.

Several risk criteria have been taken into consideration in literature, such as coherent risk measures~\cite{tamar_policy_2015}, utility functions~\cite{shen_risk-averse_2014} and return variance \citep{tamar_variance_2013} (defined in Equation~\eqref{eq:path_variance}); the latter is the most studied and the closest to reward volatility.
Typically, the risk-averse Bellman equation for the return variance defined by Sobel~\citep{sobel_variance_1982} is considered. 
Unlike Equation~\eqref{eq:bellmanV}, Sobel's equation does not satisfy the monotonicity property of dynamic programming. This has lead to a variety of complex approaches to the return-variance minimization \citep{Tamar_variance,tamar_variance_2013,pra-ghav,prashanth_actor-critic_2014}.
The popular \textit{mean-variance} objective $J_\pi - \beta \sigma_\pi^2$ is similar to the one considered in this paper (Equation ~\eqref{eq:eta}), with the return variance instead of the reward volatility. An actor-critic algorithm that considers this objective function is presented in~\citep{tamar_variance_2013,prashanth_actor-critic_2014,pra-ghav}, the latter also proposing an actor-only policy gradient. Sometimes, the objective is slightly modified by maximizing the expected value of the return while constraining the variance, such as in ~\citep{Tamar_variance}.
The \textit{Sharpe ratio} $J_\pi/{\sqrt{\sigma_\pi^2}}$ is an important financial index, which describes how much excess return you receive for the extra volatility you endure for holding a riskier asset. This ratio is used in~\citep{Tamar_variance}, where an actor-only policy gradient is proposed. 
In~\citep{moody2001learning} and~\citep{Morimura2010NonparametricRD} different approaches are used which do not need a Bellman equation for the return variance. In \citep{moody2001learning}, one of the first applications of RL to trading, the authors consider the \emph{differential Sharpe ratio} (a first order approximation of the Sharpe ratio), bypassing the direct variance calculation and proposing a policy gradient algorithm. In \citep{Morimura2010NonparametricRD}, the problem is tackled by approximating the distribution of the possible returns, deriving a Distributional Bellman Equation and minimizing the CVaR. Actor-critic algorithms for maximizing the expected return under CVaR constraints are proposed in~\citep{chow2017risk}. \\
Even if from some of these measures a Bellman Equation can be derived, unfortunately, such recursive relationships are always non-linear. This prevents one to extend the safe guarantees (Theorem 5) and the TRPO algorithm to risk measures other than reward-volatility
The second literature stream is dedicated to the safe update, which, until now, has only been defined for the standard risk-neutral objective function. The seminal paper for this setting is~\citep{kakade2002approximately}, which proposes a conservative policy iteration algorithm with monotonic improvement guarantees for mixtures of greedy policies. This approach is generalized to stationary policies in~\citep{pirotta_spi} and to stochastic policies in~\citep{schulman2015trust}. Building on the former, monotonically improving policy gradient algorithms are devised in~\citep{pirotta2013adaptive,papini2017adaptive} for Gaussian policies. Similar algorithms are developed for Lipschitz policies~\citep{pirotta2015policy} under the assumption of Lipscthiz-continuous reward and transition functions and, more recently, for smoothing policies~\citep{papini2019smoothing} on general MDPs\footnote{A safe version of VOLA-PG based on this line of work is presented in Appendix~\ref{app:safepg} as a more rigorous (albeit less practical) alternative to TRVO.}. On the other hand,~\citep{schulman2015trust} propose TRPO, a general policy gradient algorithm inspired by the monotonically-improving policy iteration strategy. Although the theoretical guarantees are formally lost due a series of approximations, TRPO enjoyed great empirical success in recent years, especially in combination with deep policies. Moreover, the \textit{exact} version of TRPO has been shown to converge to the optimal policy~\citep{neu2017unified}. Proximal Policy Optimization~\citep{schulman2017proximal} is a further approximation of TRPO which has been used to tackle complex, large-scale control problems~\citep{OpenAI_dota,Heess2017Emergence}.

\section{Experiments} \label{sec:experiments}

In this section, we show an empirical analysis of the performance of VOLA-PG (Algorithm~\ref{alg:pg}) and TRVO (Algorithm~\ref{alg:trvo}) applied to a simplified portfolio management task taken from~\citep{Tamar_variance}, and a trading task. We compare these results with two other approaches: a mean-variance policy gradient approach presented in the same article, which we adjusted to take into account discounting, and a risk averse transformation of rewards in the \textit{TRPO} algorithm.
Indeed, a possible way to control the risk of the immediate reward, is to consider its expected utility, converting the the immediate reward $R_t$ into $\widetilde{R}_t$, with a parameter $c$ controlling the sensitivity to the risk: $\widetilde{R}_t \coloneqq \frac{1-e^{-c R_t}}{c}$. This transformation generates a strong parallelism to the literature regarding the variance of the return: as the maximization of the exponential utility function applied to the return is approximately equal to the optimization of the mean-variance objective used in \citep{tamar_variance_2013}, the suggested reward transformation is similarly related to our mean-volatility objective (in a loose
approximation, as depicted in the Appendix C). By applying such a transformation, it could be possible to obtain risk-averse policies similar to ours, but running standard return-optimizing algorithms. 
However, the aforementioned approximation is sound only for small values of the risk aversion coefficient and it suffers the same limitations as each expected utility, as for example shown in \citep{Beyond_EU}. Moreover, the possibility for the agent to get negative rewards might be followed by a strong instability of the learning process. We use it as a baseline in the following experiments, with the name of TRPO-exp.
\subsection{Portfolio Management}
\paragraph{Setting}
The portfolio domain considered is composed of a liquid and a non-liquid asset. The liquid asset has a fixed interest rate $r_l$, while the non-liquid asset has a stochastic interest rate that switches between two values, $r_{nl}^{low}$ and $r^{high}_{nl}$, with probability $p_{switch}$. If the non-liquid asset does not default (a default can happen with probability $p_{risk}$), it is sold when it reaches maturity after a fixed period of $N$ time steps. At $t= 0$, the portfolio is composed only of the liquid asset. At every time step, the investor can choose to buy an amount of non-liquid assets (with a maximum of $M$ assets), each one with a fixed cost $\alpha$. Let us denote the state at time t as $\textbf{x}(t)\in\mathbb{R}^{N+2}$, where $x_1$ is the allocation in the liquid asset, $x_2,\dots,x_{N+1}\in[0,1]$ are the allocations in the non-liquid assets (with time to maturity respectively equal to $1,\dots,N$ time-steps), and $x_{N+2}(t) =r_{nl}(t)-\EV_{t'<t}[r_{nl}(t')]$. When the non-liquid asset is sold, the gain is added to the liquid asset. The reward at time $t$ is computed (unlike \citep{Tamar_variance}) as the liquid P\&L, i.e., the difference between the liquid assets at time $t$ and $t-1$. Further details on this domain are specified in the Appendix D.

\paragraph{Results}
Figure~\ref{fig:pareto_tamar} shows how the optimal solutions of the two algorithms trade off between the maximization of the expected return and the minimization of one of the two forms of variability when changing the risk aversion parameter $\lambda$. The plot on the right shows that the mean-variance frontier obtained by Algorithm \ref{alg:pg} almost coincides with the frontier obtained using the algorithm proposed in \citep{Tamar_variance}. This means that, at least in this domain, the return variance is equally reduced by optimizing either the mean-volatility or the mean-variance objectives.
Instead, from the plot on the left, we can notice that the reward-volatility is better optimized with VOLA-PG. These results are consistent with Lemma~\ref{thm:variance_ineq}.


  \begin{figure}
      \centering
      \includegraphics[width=\columnwidth]{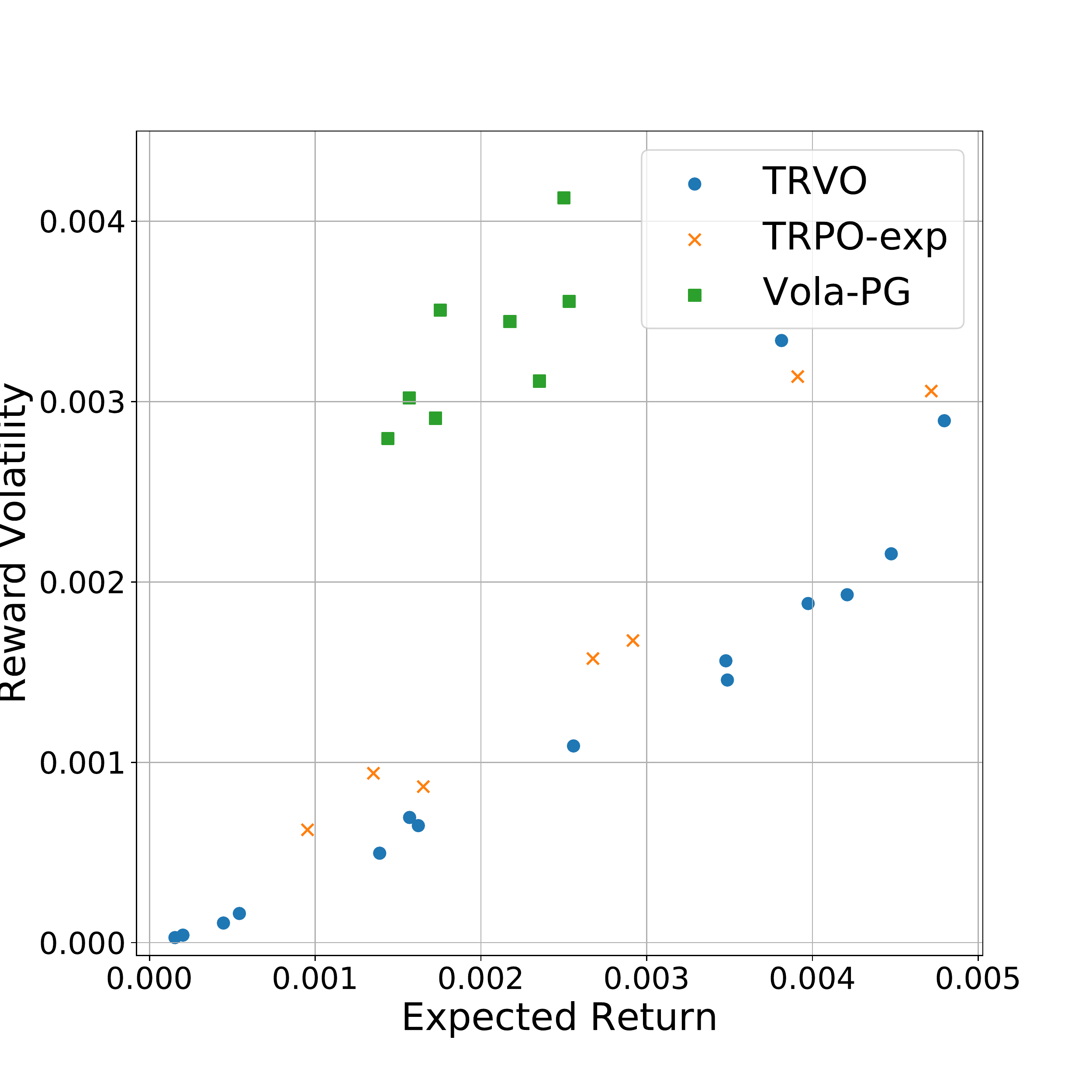}
      \caption{Trading Framework: comparison of the expected return and reward volatility obtained with three different learning algorithms: TRVO, TRPO-exp and Vola-PG.}
      \label{fig:pareto_sp}
  \end{figure}
 

\subsection{Trading Framework}
The second environment is a realistic setting: trading on the S\&P 500 index. The possible actions are buy, sell, stay flat.\footnote{This means we are assuming that short selling is possible: thus selling the stock even if you don't own it, in this case you are betting the stock price is going to decrease} Transaction costs are also considered, i.e. a fee is paid when the position changes.
The data used is the daily price of the S\&P starting from the 1980s, until 2019. The state consists of the last 10 days of percentage price changes, with the addition of the previous portfolio position as well as the fraction of episode to go, with episodes 50 days long.

\paragraph{Results}

 In Figure~\ref{fig:pareto_sp} we can see that the frontier generated by TRVO dominates the naive approach (TRPO-exp). The same Figure has VOLA-PG trained with the same number of iterations as TRVO. The obtained frontier is entirely dominated since the algorithm has not reached convergence. This reflects the faster convergence of TRPO with respect to REINFORCE. TRPO-exp becomes unstable when raising too much the risk aversion parameter $c$, so it is not possible to find the value for which the risk aversion is maximal, which is why there are no points at the bottom left.
\section{Conclusions}

Throughout this paper, we propose a combination of methodologies to control risk in a RL framework. We define a risk measure called reward volatility that captures the variability of the rewards between steps. Optimizing this measure allows to obtain smoother trajectories that avoid shocks, which is a fundamental feature in risk-averse contexts. This measure bounds the variance of the returns and it is possible to derive a linear Bellman equation for it (differently from other risk measures). We propose a policy gradient algorithm that optimizes the mean-volatility objective, and, thanks to the aforementioned linearity, we derive a trust region update that ensures a monotonic improvement of our objective at each gradient update, a fundamental characteristic for online algorithms in a real world context. 
The proposed algorithms are tested on two financial environments and TRVO was shown to outperform the naive approach on the Trading Framework and reach convergence faster than its policy gradient counterpart. Future developements could consists in testing the algorithm in a real online financial setting where new data-points do not belong to the training distribution.
To conclude, the framework is the first to take into account two kinds of safety, as it is capable of keeping risk under control while maintaining the same training and convergence properties as state-of-the-art risk-neutral approaches.


{
\bibliography{bib}

\begin{thebibliography}{}

\bibitem[\protect\citeauthoryear{Chow \bgroup et al\mbox.\egroup
  }{2017}]{chow2017risk}
Chow, Y.; Ghavamzadeh, M.; Janson, L.; and Pavone, M.
\newblock 2017.
\newblock Risk-constrained reinforcement learning with percentile risk
  criteria.
\newblock {\em JMLR} 18(1):6070--6120.

\bibitem[\protect\citeauthoryear{Deisenroth \bgroup et al\mbox.\egroup
  }{2013}]{deisenroth2013survey}
Deisenroth, M.~P.; Neumann, G.; Peters, J.; et~al.
\newblock 2013.
\newblock A survey on policy search for robotics.
\newblock {\em Foundations and Trends{\textregistered} in Robotics}
  2(1--2):1--142.

\bibitem[\protect\citeauthoryear{Di~Castro, Tamar, and
  Mannor}{2012}]{Tamar_variance}
Di~Castro, D.; Tamar, A.; and Mannor, S.
\newblock 2012.
\newblock Policy gradients with variance related risk criteria.
\newblock {\em ICML} 1.

\bibitem[\protect\citeauthoryear{Garc{{\'i}}a and
  Fern{{\'a}}ndez}{2015}]{garcia15a}
Garc{{\'i}}a, J., and Fern{{\'a}}ndez, F.
\newblock 2015.
\newblock A comprehensive survey on safe reinforcement learning.
\newblock {\em JMLR} 16:1437--1480.

\bibitem[\protect\citeauthoryear{Gosavi~et al.}{2014}]{Beyond_EU}
Gosavi~et al., A.
\newblock 2014.
\newblock Beyond exponential utility functions: A variance-adjusted approach
  for risk-averse reinforcement learning.
\newblock In {\em 2014 ADPRL},  1--8.

\bibitem[\protect\citeauthoryear{H{\"a}rdle and
  Simar}{2012}]{hardle2007applied}
H{\"a}rdle, W., and Simar, L.
\newblock 2012.
\newblock {\em Applied multivariate statistical analysis}, volume 22007.
\newblock Springer.

\bibitem[\protect\citeauthoryear{Heess \bgroup et al\mbox.\egroup
  }{2017}]{Heess2017Emergence}
Heess, N.; TB, D.; Sriram, S.; Lemmon, J.; Merel, J.; Wayne, G.; Tassa, Y.;
  Erez, T.; Wang, Z.; Eslami, S. M.~A.; Riedmiller, M.~A.; and Silver, D.
\newblock 2017.
\newblock Emergence of locomotion behaviours in rich environments.
\newblock {\em CoRR} abs/1707.02286.

\bibitem[\protect\citeauthoryear{Kakade and
  Langford}{2002}]{kakade2002approximately}
Kakade, S., and Langford, J.
\newblock 2002.
\newblock Approximately optimal approximate reinforcement learning.
\newblock In {\em ICML}, volume~2,  267--274.

\bibitem[\protect\citeauthoryear{Moldovan and
  Abbeel}{2012}]{moldovan_risk_2012}
Moldovan, T.~M., and Abbeel, P.
\newblock 2012.
\newblock Risk aversion in {Markov} decision processes via near optimal
  {Chernoff} bounds.
\newblock In {\em NeurIPS},  3131--3139.

\bibitem[\protect\citeauthoryear{Moody and Saffell}{2001}]{moody2001learning}
Moody, J., and Saffell, M.
\newblock 2001.
\newblock Learning to trade via direct reinforcement.
\newblock {\em IEEE transactions on neural Networks} 12(4):875--889.

\bibitem[\protect\citeauthoryear{Morimura \bgroup et al\mbox.\egroup
  }{2010}]{Morimura2010NonparametricRD}
Morimura, T.; Sugiyama, M.; Kashima, H.; Hachiya, H.; and Tanaka, T.
\newblock 2010.
\newblock Nonparametric return distribution approximation for reinforcement
  learning.
\newblock In {\em ICML}.

\bibitem[\protect\citeauthoryear{Neu, Jonsson, and
  G{\'{o}}mez}{2017}]{neu2017unified}
Neu, G.; Jonsson, A.; and G{\'{o}}mez, V.
\newblock 2017.
\newblock A unified view of entropy-regularized markov decision processes.
\newblock {\em CoRR} abs/1705.07798.

\bibitem[\protect\citeauthoryear{OpenAI}{2018}]{OpenAI_dota}
OpenAI.
\newblock 2018.
\newblock Openai five.
\newblock \url{https://blog.openai.com/openai-five/}.

\bibitem[\protect\citeauthoryear{Papini, Pirotta, and
  Restelli}{2017}]{papini2017adaptive}
Papini, M.; Pirotta, M.; and Restelli, M.
\newblock 2017.
\newblock Adaptive batch size for safe policy gradients.
\newblock In {\em NeurIPS},  3591--3600.

\bibitem[\protect\citeauthoryear{Papini, Pirotta, and
  Restelli}{2019}]{papini2019smoothing}
Papini, M.; Pirotta, M.; and Restelli, M.
\newblock 2019.
\newblock Smoothing policies and safe policy gradients.

\bibitem[\protect\citeauthoryear{Peters and
  Schaal}{2008}]{peters2008reinforcement}
Peters, J., and Schaal, S.
\newblock 2008.
\newblock Reinforcement learning of motor skills with policy gradients.
\newblock {\em Neural networks} 21(4):682--697.

\bibitem[\protect\citeauthoryear{Pirotta \bgroup et al\mbox.\egroup
  }{2013}]{pirotta_spi}
Pirotta, M.; Restelli, M.; Pecorino, A.; and Calandriello, D.
\newblock 2013.
\newblock Safe policy iteration.
\newblock In {\em ICML},  307--315.

\bibitem[\protect\citeauthoryear{Pirotta, Restelli, and
  Bascetta}{2013}]{pirotta2013adaptive}
Pirotta, M.; Restelli, M.; and Bascetta, L.
\newblock 2013.
\newblock Adaptive step-size for policy gradient methods.
\newblock In {\em NeurIPS 26}. Curran Associates, Inc.
\newblock  1394--1402.

\bibitem[\protect\citeauthoryear{Pirotta, Restelli, and
  Bascetta}{2015}]{pirotta2015policy}
Pirotta, M.; Restelli, M.; and Bascetta, L.
\newblock 2015.
\newblock Policy gradient in lipschitz markov decision processes.
\newblock {\em Machine Learning} 100(2-3):255--283.

\bibitem[\protect\citeauthoryear{Prashanth and
  Ghavamzadeh}{2014a}]{prashanth_actor-critic_2014}
Prashanth, L.~A., and Ghavamzadeh, M.
\newblock 2014a.
\newblock Actor-critic algorithms for risk-sensitive reinforcement learning.
\newblock {\em arXiv preprint arXiv:1403.6530}.

\bibitem[\protect\citeauthoryear{Prashanth and Ghavamzadeh}{2014b}]{pra-ghav}
Prashanth, L.~A., and Ghavamzadeh, M.
\newblock 2014b.
\newblock Variance-constrained actor-critic algorithms for discounted and
  average reward mdps.
\newblock {\em CoRR} abs/1403.6530.

\bibitem[\protect\citeauthoryear{Schulman \bgroup et al\mbox.\egroup
  }{2015}]{schulman2015trust}
Schulman, J.; Levine, S.; Abbeel, P.; Jordan, M.~I.; and Moritz, P.
\newblock 2015.
\newblock Trust region policy optimization.
\newblock In {\em {ICML}},  1889--1897.

\bibitem[\protect\citeauthoryear{Schulman \bgroup et al\mbox.\egroup
  }{2017}]{schulman2017proximal}
Schulman, J.; Wolski, F.; Dhariwal, P.; Radford, A.; and Klimov, O.
\newblock 2017.
\newblock Proximal policy optimization algorithms.
\newblock {\em CoRR} abs/1707.06347.

\bibitem[\protect\citeauthoryear{{Shen} \bgroup et al\mbox.\egroup
  }{2014}]{shen_risk-averse_2014}
{Shen}, Y.; {Huang}, R.; {Yan}, C.; and {Obermayer}, K.
\newblock 2014.
\newblock Risk-averse reinforcement learning for algorithmic trading.
\newblock In {\em 2014 IEEE Conference on Computational Intelligence for
  Financial Engineering Economics (CIFEr)},  391--398.

\bibitem[\protect\citeauthoryear{Sobel}{1982}]{sobel_variance_1982}
Sobel, M.~J.
\newblock 1982.
\newblock The variance of discounted {Markov} decision processes.
\newblock {\em Journal of Applied Probability} 19(4):794--802.

\bibitem[\protect\citeauthoryear{Sutton and Barto}{1998}]{Suttonbook:1998}
Sutton, R.~S., and Barto, A.~G.
\newblock 1998.
\newblock {\em Introduction to Reinforcement Learning}.
\newblock Cambridge, MA, USA: MIT Press, 1st edition.

\bibitem[\protect\citeauthoryear{Sutton \bgroup et al\mbox.\egroup
  }{2000}]{sutton2000policy}
Sutton, R.~S.; McAllester, D.~A.; Singh, S.~P.; and Mansour, Y.
\newblock 2000.
\newblock Policy gradient methods for reinforcement learning with function
  approximation.
\newblock In {\em NeurIPS},  1057--1063.

\bibitem[\protect\citeauthoryear{Tamar and Mannor}{2013}]{tamar_variance_2013}
Tamar, A., and Mannor, S.
\newblock 2013.
\newblock Variance adjusted actor critic algorithms.
\newblock {\em arXiv preprint arXiv:1310.3697}.

\bibitem[\protect\citeauthoryear{Tamar \bgroup et al\mbox.\egroup
  }{2015}]{tamar_policy_2015}
Tamar, A.; Chow, Y.; Ghavamzadeh, M.; and Mannor, S.
\newblock 2015.
\newblock Policy {Gradient} for {Coherent} {Risk} {Measures}.
\newblock {\em CoRR} ~9.

\bibitem[\protect\citeauthoryear{Tamar \bgroup et al\mbox.\egroup
  }{2017}]{tamar_sequential_2017}
Tamar, A.; Chow, Y.; Ghavamzadeh, M.; and Mannor, S.
\newblock 2017.
\newblock Sequential {Decision} {Making} {With} {Coherent} {Risk}.
\newblock {\em IEEE Transactions on Automatic Control} 62(7):3323--3338.

\end{thebibliography}
\bibliographystyle{aaai}
}

\onecolumn\clearpage
\newpage
\appendix
\section{Proofs}\label{app:proofs}
First, let us define the state-occupancy measure more formally.
The $t$-step marginal state transition density under policy $\pi$ is defined recursively as follows:
\begin{align*}
	&p_{\pi}(s\xrightarrow{1}s') \coloneqq p_{\pi}(s'\vert s) \coloneqq \EV_{a\sim\pi(\cdot|s)}\left[\mathcal{P}(s'|s,a)\right],\\
	&p_{\pi}(s\xrightarrow{t+1}s') \coloneqq \int_{\Sspace}p_{\pi}(\tilde{s}\vert s)p_{\pi}(\tilde{s}\xrightarrow{t}s')\de \tilde{s}.
\end{align*}
This allows to define the following state-occupancy densities:
\begin{align*}
	d_{\pi}(s'\vert s) &\coloneqq (1-\gamma)\sum_{t=0}^{\infty}\gamma^tp_{\pi}(s \xrightarrow{t} s'); \\
	d_{\Init,\pi}(s) &\coloneqq \int_{S}\mu(s_0)d_{\pi}(s\vert s_0)\de s_0.\label{def:dmupi}
\end{align*}
measuring the discounted probability of visiting a state starting from another state or from the start, respectively.

\subsection{Variance inequality}
\variance*
\begin{proof}
Taking the left hand side (Equation~\ref{eq:path_variance}) and expanding the square we obtain\footnote{To shorten the notation, $R_t$ is used instead of $R(s_t, a_t)$.}:
\begin{equation*}
\begin{aligned}
     \sigma_\pi^2 &  =\EV_{\substack{s_0 \sim \mu \\a_t\sim\pi_{\vtheta}(\cdot|s_t)\\ s_{t+1}\sim \mathcal{P}(\cdot|s_t,a_t)}}\left[\left(\sum_{t=0}^\infty\gamma^t R_t\right)^2\right] + \frac{J_\pi^2}{(1-\gamma)^2} -\frac{2J_\pi}{(1-\gamma)}\EV_{\substack{s_0 \sim \mu \\a_t\sim\pi_{\vtheta}(\cdot|s_t)\\ s_{t+1}\sim \mathcal{P}(\cdot|s_t,a_t)}}\left[\sum_{t=0}^\infty\gamma^t R_t\right] \\&
    =\EV_{\substack{s_0 \sim \mu \\a_t\sim\pi_{\vtheta}(\cdot|s_t)\\ s_{t+1}\sim \mathcal{P}(\cdot|s_t,a_t)}} \left[\left(\sum_{t=0}^\infty\gamma^t R_t\right)^2\right] - \frac{J_\pi^2}{(1-\gamma)^2}.
\end{aligned}
\end{equation*}
Similarly, for the right hand side of the inequality:
\begin{equation*}
\begin{aligned}
 \nu^2_\pi & = \EV_{\substack{s\sim d_{\mu,\pi}\\a\sim\pi_{\vtheta}(\cdot|s)}}\big[\big(R(s,a)-J_\pi\big)^2\big] \\&
= \EV_{\substack{s\sim d_{\mu,\pi}\\a\sim\pi_{\vtheta}(\cdot|s)}}\big[R(s,a)^2\big]+J_\pi^2 - 2J_\pi\EV_{\substack{s\sim d_{\mu,\pi}\\a\sim\pi_{\vtheta}(\cdot|s)}} \big[R(s,a)\big] \\&
= \EV_{\substack{s\sim d_{\mu,\pi}\\a\sim\pi_{\vtheta}(\cdot|s)}}\big[R(s,a)^2\big]-J_\pi^2.
\end{aligned}
\end{equation*}
Thus, the inequality we want to prove reduces to:
\begin{equation*}
(1-\gamma)^2\EV_{\substack{s_0 \sim \mu \\a_t\sim\pi_{\vtheta}(\cdot|s_t)\\ s_{t+1}\sim \mathcal{P}(\cdot|s_t,a_t)}} \left[\left(\sum_{t=0}^\infty\gamma^t R_t\right)^2\right] \leq \EV_{\substack{s\sim d_{\mu,\pi}\\a\sim\pi_{\vtheta}(\cdot|s)}}\big[R(s,a)^2\big].
\end{equation*}
Consider the left hand side. By the Cauchy-Schwarz inequality, it reduces to:
\begin{align*}
    (1-\gamma)^2\EV_{\substack{s_0 \sim \mu \\a_t\sim\pi_{\vtheta}(\cdot|s_t)\\ s_{t+1}\sim \mathcal{P}(\cdot|s_t,a_t)}} \left[\left(\sum_{t=0}^\infty\gamma^t R_t\right)^2\right] \ & \leq
    (1-\gamma)^2\EV_{\substack{s_0 \sim \mu \\a_t\sim\pi_{\vtheta}(\cdot|s_t)\\ s_{t+1}\sim \mathcal{P}(\cdot|s_t,a_t)}} \left[\left(\sum_{t=0}^\infty\gamma^t\right)\left(\sum_{t=0}^\infty\gamma^t R^2_t\right)\right]\\ & 
    = (1-\gamma)\EV_{\substack{s_0 \sim \mu \\a_t\sim\pi_{\vtheta}(\cdot|s_t)\\ s_{t+1}\sim \mathcal{P}(\cdot|s_t,a_t)}} \left[\sum_{t=0}^\infty\gamma^t R^2_t\right] \\&
   = \EV_{\substack{s\sim d_{\mu,\pi}\\a\sim\pi_{\vtheta}(\cdot|s)}}\big[R(s,a)^2\big].
\end{align*}
\end{proof}
\subsection{Risk-averse policy gradient theorem}

\gradient*
\begin{proof}\\
First, we need the following property (see \eg Lemma 1 in~\cite{papini2019smoothing}): 
\begin{restatable}[]{lemma}{recursion_appendix}
    \label{lem:rec}
	Any integrable function $f:\Sspace\to\Reals$ that can be recursively defined as:
	\begin{align*}
	f(s) = g(s) + \gamma\int_{\Sspace}P_{\pi}(s'\vert s)f(s')\de s',
	\end{align*}
	where $g:\Sspace\to\Reals$ is any integrable function, is equal to:
	\begin{align*}
	f(s)  = \frac{1}{1-\gamma}\int_{\Sspace}d_{\pi}(s'\vert s){g(s')}\de s'.
	\end{align*}
\end{restatable}
From Equation~\ref{eq:bellmanV} and the definition of $W_{\pi}$, we have\footnote{To simplify the notation, the dependency on $\vtheta$ is left implicit.}:
\begin{align*}
        &X_\pi(s,a)= \big(R(s,a)-J_\pi\big)^2 + \gamma \EV_{\substack{s'\sim P(\cdot|s,a)\\a'\sim\pi_{\vtheta}(\cdot|s')}}\big[X_\pi(s',a')\big] \\
            &\qquad= \big(R(s,a)-J_\pi\big)^2 + \gamma \EV_{\substack{s'\sim P(\cdot|s,a)\\}}\big[W_\pi(s')\big], \\
        &W_\pi(s) = \EV_{a\sim\pi_{\vtheta}(\cdot|s)}\big[ \big(R(s,a)-J_\pi\big)^2\big] + \gamma \EV_{s'\sim P^\pi(\cdot|s)}\big[W_\pi(s')\big] \\
        &\qquad= \frac{1}{1-\gamma}\EV_{\substack{s'\sim d_{\pi}(\cdot|s)\\a\sim\pi_{\vtheta}(\cdot|s')}}\bigg[\big(R(s',a)-J_\pi\big)^2\bigg], \\
        &\nu^2_\pi = \EV_{\substack{s\sim d_{\mu,\pi}\\a\sim\pi_{\vtheta}(\cdot|s)}}\bigg[\big(R(s,a)-J_\pi\big)^2\bigg] \\
        &\qquad= (1-\gamma) \EV_{s\sim\mu}\bigg[W_\pi(s)\bigg].
\end{align*}
For the second part, we follow a similar argument as in \citep{sutton2000policy}.
We first consider the gradient of $X_\pi(s,a)$ and $W_\pi(s)\ \forall s,a\in\Sspace\times\Aspace$:
\begin{align*}
    &\nabla X_\pi(s,a) = -2\big(R(s,a)-J_\pi\big)\nabla J_\pi + \gamma \EV_{s'\sim P}\big[\nabla W_\pi(s')\big], \nonumber\\
    &\nabla W_\pi(s) = \nabla \int_\Aspace \pi_{\vtheta}(a|s)X_\pi(s,a)\de a \nonumber\\
                &\qquad= \int_\Aspace\big[\nabla\pi_{\vtheta}(a|s)X_\pi(s,a)+\pi_{\vtheta}(a|s)\nabla X_\pi(s,a)\big]\de a \nonumber\\
                &\qquad= \int_\Aspace\big[\nabla\pi_{\vtheta}(a|s)X_\pi(s,a)-2\pi_{\vtheta}(a|s)\big(R(s,a)-J_\pi\big)\nabla J_\pi\big]\de a \nonumber\\&\qquad\qquad+ \gamma \int_\Sspace p^\pi(s'|s)\nabla W_\pi(s')\de s'\nonumber \\
                &\qquad= \frac{1}{1-\gamma}\int_\Sspace 
                d_\pi(s'|s)\int_\Aspace\big[\nabla\pi_{\vtheta}(a|s')X_\pi(s,a)-2\pi_{\vtheta}(a|s)\big(R(s,a)-J_\pi\big)\nabla J_\pi \big] \de a \de s',\\
\end{align*}
where the last equality is from Lemma~\ref{lem:rec}. Finally:
\begin{align*}
    &\nabla \nu^2_\pi = (1-\gamma)\int_\Sspace \mu(s)\nabla W_\pi(s)\de s \nonumber\\
                    &\qquad= \int_\Sspace 
                d_{\mu,\pi}(s)\int_\Aspace\big[\nabla\pi_{\vtheta}(a|s)X_\pi(s,a)-2\pi_{\vtheta}(a|s)\big(R(s,a)-J_\pi\big)\nabla J_\pi \big] \de a \de s \nonumber\\
                &\qquad= \EV_{\substack{s'\sim d_{\mu,\pi}\\a\sim\pi_{\vtheta}(\cdot|s)}}\big[\nabla\log\pi_{\vtheta}(a|s)X_\pi(s,a)\big]- 2\nabla J_\pi\EV_{\substack{s'\sim d_{\mu,\pi}\\a\sim\pi_{\vtheta}(\cdot|s)}}\big[R(s,a)-J_\pi\big] \nonumber \\
                &\qquad=\EV_{\substack{s'\sim d_{\mu,\pi}\\a\sim\pi_{\vtheta}(\cdot|s)}}\big[\nabla\log\pi_{\vtheta}(a|s)X_\pi(s,a)\big].\nonumber
\end{align*}
\end{proof}

\subsection{Trust Region Volatility Optimization}
\ADVlambda*
\begin{proof}\footnote{For the sake of clarity, we use the notation $\tau|\pi$ to denote the expectation over trajectories: $s_0\sim\mu, a_t\sim\pi(\cdot|s_t), s_{t+1}\sim \mathcal{P}(\cdot|s_t,a_t)$.}
\begin{align*}
\eta_{\widetilde{\pi}} &= (1-\gamma)\EV_{\tau|\widetilde{\pi}}\big[\sum_t\gamma^t(R(s_t,a_t)-\lambda(R(s_t,a_t)-J_{\widetilde{\pi}})^2)\big]\\
&= (1-\gamma)\EV_{\tau|\widetilde{\pi}}\big[V^\lambda_\pi(s_0)-V^\lambda_\pi(s_0)+\sum_t\gamma^t(R(s_t,a_t)-\lambda(R(s_t,a_t)-J_{\widetilde{\pi}})^2)\big]\\
&=\eta_\pi + (1-\gamma)\EV_{\tau|\widetilde{\pi}}\big[\sum_t\gamma^t(R(s_t,a_t)-\lambda(R(s_t,a_t)-J_{\widetilde{\pi}})^2+\gamma V_\pi(s_{t+1})-V_\pi(s_t))\big]
\end{align*}
Now, the goal is to obtain the discounted sum of the mean-volatility advantages defined in Eqaution~\eqref{eq:B-dvantage}; however, it must be evaluated using policy $\pi$ instead of $J_{\widetilde{\pi}}$. Hence, we recall the result in \citep{kakade2002approximately}\footnote{the difference is only in the normalization terms, which are used accordingly to the definitions above.}:
\begin{equation*}
    J_{\widetilde{\pi}} = J_\pi + (1-\gamma)\EV_{\tau|\widetilde{\pi}}\big[\sum_t \gamma^t A_\pi(s_t, a_t)\big]
\end{equation*}
Hence, by using $R_t$ to denote $R(s_t,a_t)$:
\begin{align*}
(R_t-J_{\widetilde{\pi}})^2 &= (R_t-J_{\pi}-(1-\gamma)\EV_{\tau|\widetilde{\pi}}\big[\sum_t \gamma^t A_\pi(s_t, a_t)\big])^2\\
&= (R_t-J_\pi)^2 + (1-\gamma)^2\EV_{\tau|\widetilde{\pi}}\big[\sum_t \gamma^t A_\pi(s_t, a_t)\big]^2-2(1-\gamma)(R_t-J_\pi)\EV_{\tau|\widetilde{\pi}}\big[\sum_t \gamma^t A_\pi(s_t, a_t)\big]
\end{align*}
In this way, it is possible to separate the mean-volatility advantage function from the standard advantage function, since the performance difference becomes:
\begin{align*}
    \eta_{\widetilde{\pi}}-\eta_\pi =& (1-\gamma)\EV_{\tau|\widetilde{\pi}}\big[\sum_t\gamma^t(R_t-\lambda(R_t-J_\pi)^2+\gamma V^\lambda_\pi(s_{t+1})-V^\lambda_\pi(s_t))\big] \\
   & -\lambda (1-\gamma)^3\EV_{\tau|\widetilde{\pi}}\big[\sum_t\gamma^t \EV_{\tau|\widetilde{\pi}}[\sum_t \gamma^t A_\pi(s_t, a_t)]^2\big] \\
    &+ 2\lambda (1-\gamma)^2\EV_{\tau|\widetilde{\pi}}\big[\sum_t\gamma^t(R_t-J_\pi) \EV_{\tau|\widetilde{\pi}}[\sum_t \gamma^t A_\pi(s_t, a_t)]\big]\\
    =& (1-\gamma)\EV_{\tau|\widetilde{\pi}}\big[\sum_t\gamma^tA_\pi^\lambda(s_t,a_t)]-\lambda (1-\gamma)^2\EV_{\tau|\widetilde{\pi}}[\sum_t \gamma^t A_\pi(s_t, a_t)]^2\\
    &+ 2\lambda(1-\gamma)^2\EV_{\tau|\widetilde{\pi}}[\sum_t \gamma^t A_\pi(s_t, a_t)]\EV_{\tau|\widetilde{\pi}}[\sum_t \gamma^t(R_t- J_\pi)]
\end{align*}
But, if we consider that
\begin{align*}
    \EV_{\tau|\widetilde{\pi}}[\sum_t \gamma^t(R_t-J_\pi)] &= \EV_{\tau|\widetilde{\pi}}[\sum_t\gamma^tR_t]-\frac{J_\pi}{1-\gamma} \\
    &= \frac{J_{\widetilde{\pi}}-J_\pi}{1-\gamma} = \EV_{\tau|\widetilde{\pi}}[\sum_t\gamma^tA_\pi(s_t, a_t)]
\end{align*}
Then, collecting everything together, we obtain
\begin{align*}
    \eta_{\widetilde{\pi}}-\eta_\pi = (1-\gamma)\EV_{\tau|\widetilde{\pi}}\big[\sum_t\gamma^tA_\pi^\lambda(s_t,a_t)]+\lambda (1-\gamma)^2\EV_{\tau|\widetilde{\pi}}[\sum_t \gamma^t A_\pi(s_t, a_t)]^2
\end{align*}
\end{proof}

\subsection{Gradient Estimator}\label{ssec:gradest}
Here we are in a finite setting, let's define a sampled trajectory as $\tau^i_j = s_0^i, a_0^i, ... , s^i_{T-1}, a^i_{T-1}$ with $i = 0,\dots,N-1$, and $\tau_j^i \in \mathcal{D}_j$.

\xestimator*
\begin{proof}
First of all, we recall that 
\begin{equation*}
     \EV_\tau[\hat{J}] = J_\pi.
\end{equation*}
Thus:
\begin{align*}
    \EV_{\tau_1}\EV_{\tau_2}\EV_{\substack{s'\sim d_\pi\\a'\sim\pi_{\vtheta}(\cdot|s')}}[\widehat{X}]  &= \  \frac{1-\gamma}{1-\gamma^{T}}\frac{1}{N} \sum_{i=0}^{N-1} \EV_{\tau_1}\EV_{\tau_2}\EV_{\substack{s'\sim d_\pi\\a'\sim\pi_{\vtheta}(\cdot|s')}} \left[ \sum_{t=0}^{T-1} \gamma^t (\mathcal{R}(s_t^i,a_t^i)-\hat{J}_1)(\mathcal{R}(s_t^i,a_t^i)-\hat{J}_2)\right]\\ &=\frac{1-\gamma}{1-\gamma^{T}}\frac{1}{N} \sum_{i=0}^{N-1}\EV_{\substack{s'\sim d_\pi\\a'\sim\pi_{\vtheta}(\cdot|s')}} \sum_{t=0}^{T-1} \gamma^t \left[\EV_{\tau_1}(\mathcal{R}(s_t^i,a_t^i)-\hat{J}_1)\EV_{\tau_2}(\mathcal{R}(s_t^i,a_t^i)-\hat{J}_2)\right]
    \\  &= \frac{1-\gamma}{1-\gamma^{T}}\frac{1}{N}\sum_{i=0}^{N-1} \sum_{t=0}^{T-1} \gamma^t \EV_{\substack{s'\sim d_\pi\\a'\sim\pi_{\vtheta}(\cdot|s')}}\left[(\mathcal{R}(s_t^i,a_t^i)-J_\pi)(\mathcal{R}(s_t^i,a_t^i)-J_\pi)\right]\\ &= 
    \frac{1-\gamma}{1-\gamma^{T}}\frac{1}{N} \sum_{i=0}^{N-1} \sum_{t=0}^{T-1} \gamma^t \EV_{\substack{s'\sim d_\pi\\a'\sim\pi_{\vtheta}(\cdot|s')}}\left[(\mathcal{R}(s'_t,a'_t)-J_\pi)^2\right]\\&
    = X_\pi.
\end{align*}
\end{proof}

\section{Safe Volatility Optimization}\label{app:safepg}
In this section, we provide a more rigorous alternative to TRVO. To do so, we simply adapt the Safe Policy Gradient approach from~\cite{papini2019smoothing} to our mean-volatility objective and find safe, adaptive values for the step size $\alpha$ and the batch size $N$ in Algorithm~\ref{alg:pg}.
We restrict our analysis to \textit{smoothing} policies:
\begin{restatable}[Smoothing policies]{definition}{fosmooth}\label{def:smooth}
	Let $\Pi_{\Theta}=\{\pi_{\vtheta}\mid \vtheta\in\Theta\subseteq\Reals^m\}$ be a class of twice-differentiable parametric policies. We call it \textbf{smoothing} if the parameter space $\Theta$ is convex and there exists a set of non-negative constants  $(\psi;\kappa;\xi)$ such that, for each state and in expectation over actions, the Euclidean norm of the score function, its square Euclidean norm and the spectral norm of the observed information are upper-bounded, i.e., $\forall s \in \Sspace$:
	\begin{align*}
	\EV_a\big[\|\nabla\log\pi_{\vtheta}(a|s)\|\big]\leq\sm; \qquad
    \EV_a\big[\|\nabla\log\pi_{\vtheta}(a|s)\|^2\big]\leq\kappa;  
    \qquad
    \EV_a\big[\|\nabla\nabla^\top\log\pi_{\vtheta}(a|s)\|\big]\leq\xi. 
	\end{align*}
\end{restatable}
Smoothing policies include Gaussians with fixed variance and Softmax policies~\citep{papini2019smoothing}. For smoothing policies, the performance improvement yielded by a generic parameter update is lower bounded by:

\begin{restatable}[]{theorem}{safeimprovement}\label{thm:safe_improvement}
Let $\Pi_{\Theta}$ be a smoothing policy class, $\vtheta\in\Theta$ and $\vtheta' = \vtheta+\Delta\vtheta$. For any $\Delta\vtheta\in\Reals^m$:
\begin{equation}\label{eq:update}
    \eta_{\vtheta'}-\eta_{\vtheta} \geq \langle \Delta\vtheta, \nabla\eta_{\vtheta}\rangle - \frac{L}{2}\|\Delta\vtheta\|^2,
\end{equation}
where:
\begin{equation*}
        L= \frac{c}{(1-\gamma)^2}\left(\frac{2\gamma\psi^2}{1-\gamma}+\kappa+\xi\right)+\frac{2R^2_{max}\psi^2}{(1-\gamma)^3}.
\end{equation*}
\end{restatable}

To prove this result, we need some additional Lemmas. The challenging part is bounding the spectral norm of the Hessian Matrix of $\eta$. First, we derive a compact expression for the Hessian of $\nu$:
\begin{restatable}[]{lemma}{hessian_appendix}
Given a twice-differentiable parametric policy $\pi_{\vtheta}$, the policy Hessian is:
\begin{equation*}
\begin{aligned}
    \mathcal{H}\nu^2_\pi = \frac{1}{1-\gamma}\EV_{\substack{s\sim d_{\mu,\pi}\\ a\sim\pi_{\vtheta}(\cdot|s')}}\bigg[&\big(\nabla\log\pi_{\vtheta}(a|s)\nabla^\top\log\pi_{\vtheta}(a|s)+\mathcal{H}\log\pi_{\vtheta}(a|s)\big)X_\pi(s,a)\\
    &+\nabla\log\pi_{\vtheta}(a|s)\nabla^\top X_\pi(s,a) + \nabla X_\pi(s,a)\nabla^\top\log\pi_{\vtheta}(a|s)\bigg]\\ &+\frac{2}{1-\gamma}\nabla J \nabla^\top J.
    \end{aligned}
\end{equation*}
\end{restatable}

\begin{proof}
First note that:
\begin{equation*}
    \mathcal{H}(\pi_{\vtheta}  X_\pi ) =  X_\pi \mathcal{H}\pi_{\vtheta}  + \nabla\pi_{\vtheta} \nabla^\top X_\pi  + \nabla X_\pi  \nabla^\top\pi_{\vtheta}  + \pi_{\vtheta} \mathcal{H}X_\pi.
\end{equation*}
Then\footnote{For the sake of brevity, the dependence on $s,a$ is often omitted.}:
\begin{align*}
    \mathcal{H}W_\pi(s) = \mathcal{H}\int_{\Aspace}&\pi_{\vtheta}(a|s)X_\pi(s,a)da \\ =\int_\Aspace&\bigg[X_\pi(s,a) \mathcal{H}\pi_{\vtheta}(a|s)  + \nabla\pi_{\vtheta}(a|s) \nabla^\top X_\pi(s,a)  \\&+ \nabla X_\pi(s,a)  \nabla^\top\pi_{\vtheta}(a|s)  + \pi(a|s) \mathcal{H}X_\pi(s,a)\bigg]da \\
    = \int_\Aspace&\bigg[X\mathcal{H}\pi_{\vtheta}  + \nabla\pi_{\vtheta} \nabla^\top X  + \nabla X \nabla^\top\pi_{\vtheta}  + \pi_{\vtheta}\mathcal{H}\big((R-J_\pi)^2\big)\bigg]da \\&+ \gamma\int_\Sspace P(s'|s)W_\pi(s')ds'.
\end{align*}
Since:
\begin{align*}
    \mathcal{H}\bigg((R(s,a)-J_\pi)^2\bigg) &= 2\big(R(s,a)-J_\pi\big)\mathcal{H}[R(s,a)-J_\pi]+2\nabla J_\pi\nabla^\top J_\pi \\
    &= -2\mathcal{H} J_\pi \big(R(s,a)-J_\pi\big) +2\nabla J_\pi\nabla^\top J_\pi,
\end{align*}
then, using Lemma \ref{lem:rec} and the log-trick:
\begin{align*}
    &\mathcal{H}W_\pi(s)=  \frac{1}{1-\gamma}\EV_{s'\sim d_{\pi}(\cdot|s)}\big[\int_\Aspace[X\mathcal{H}\pi_{\vtheta}  + \nabla\pi_{\vtheta} \nabla^\top X  + \nabla X \nabla^\top\pi_{\vtheta}   -2\pi_{\vtheta}\mathcal{H} J_\pi\big(R-J_\pi\big)]da\big]\\
    &\qquad\qquad\qquad+ 2\nabla J_\pi\nabla^\top J_\pi \int_\Sspace d(s'|s)\int_\Aspace\pi_{\vtheta}(a|s')\de a\de s' \\
    &\qquad= \frac{1}{1-\gamma}\EV_{\substack{s'\sim d_{\pi}(\cdot|s)\\a\sim\pi_{\vtheta}(\cdot|s')}}\hspace{-4pt}\big[(\nabla\log\pi_{\vtheta}\nabla^\top\log\pi_{\vtheta}\hspace{-1pt}+\hspace{-1pt}\mathcal{H}\log\pi_{\vtheta})X+\hspace{-2pt}\nabla\log\pi_{\vtheta}\nabla^\top X +\hspace{-2pt}\nabla X \nabla^\top\log\pi_{\vtheta}\big]\\
    &\qquad\qquad-\frac{2}{1-\gamma}\mathcal{H} J_\pi \EV_{\substack{s\sim d_{\pi}(\cdot|s)\\a\sim\pi_{\vtheta}(\cdot|s')}}\big[R-J_\pi\big]+\frac{2}{1-\gamma}\nabla J_\pi\nabla^\top J_\pi,\\
&\mathcal{H}\nu_{\pi}^2 = \EV_{s\sim\mu}\big[\mathcal{H}W_\pi(s)\big] \\
&\qquad=  \frac{1}{1-\gamma}\EV_{\substack{s\sim d_{\mu,\pi}\\a\sim\pi_{\vtheta}}}\hspace{-9pt}\big[(\nabla\log\pi_{\vtheta}\nabla^\top\log\pi_{\vtheta}+\mathcal{H}\log\pi_{\vtheta})X+\nabla\log\pi_{\vtheta}\nabla^\top X +\nabla X\nabla^\top\log\pi_{\vtheta}\big] \\
&\qquad\qquad-\frac{2}{1-\gamma}\mathcal{H} J_\pi \EV_{\substack{s\sim d_{\mu,\pi}(\cdot|s)\\a\sim\pi_{\vtheta}(\cdot|s')}}\big[R-J_\pi\big]+\frac{2}{1-\gamma}\nabla J_\pi\nabla^\top J_\pi.
\end{align*}
The second term is null, from the definition \eqref{eq:J_hat} of $J_\pi$:
\begin{equation*}
    \EV_{\substack{s\sim d_{\mu,\pi}\\a\sim\pi_{\vtheta}(\cdot|s')}}\big[R(s,a)-J_\pi\big] = 0.
\end{equation*}
\end{proof}
It is now possible to use the results proven in \citep{papini2019smoothing} for the Hessian of $J_\pi$:
\begin{equation*}
\begin{aligned}
    \mathcal{H}J_\pi = \frac{1}{1-\gamma}\EV_{\substack{s\sim d_{\mu,\pi}\\a\sim\pi_{\vtheta}(\cdot|s')}}\bigg[&\big(\nabla\log\pi_{\vtheta}(a|s)\nabla^\top\log\pi_{\vtheta}(a|s)+\mathcal{H}\log\pi_{\vtheta}(a|s)\big)Q_\pi(s,a) \\&+\nabla\log\pi_{\vtheta}(a|s)\nabla^\top Q_\pi(s,a) +\nabla Q_\pi(s,a)\nabla^\top\log\pi_{\vtheta}(a|s) \bigg].
\end{aligned}
\end{equation*}
Putting everything together, the following holds:
\begin{equation}\label{eq:Heta}
\begin{aligned}
    \mathcal{H}\eta_\pi = &\frac{1}{1-\gamma}\EV_{\substack{s\sim d_{\mu,\pi}\\a\sim\pi_{\vtheta}(\cdot|s')}}\bigg[\big(\nabla\log\pi_{\vtheta}(a|s)\nabla^\top\log\pi_{\vtheta}(a|s)+\mathcal{H}\log\pi_{\vtheta}(a|s)\big)\big[Q_\pi(s,a)-\lambda X_\pi(s,a)\big] \\&+\nabla\log\pi_{\vtheta}(a|s)\nabla^\top\big[Q_\pi(s,a)-\lambda X_\pi(s,a)\big] +\nabla \big[Q_\pi(s,a)-\lambda X_\pi(s,a)\big]\nabla^\top\log\pi_{\vtheta}(a|s) \bigg]\\
    &+\frac{2}{1-\gamma}\nabla J_\pi\nabla^\top J_\pi.
\end{aligned}
\end{equation}

This expression allows to upper bound the spectral norm of the Hessian: 

\begin{restatable}[]{lemma}{hessian_bound_appendix}\label{lem:hba}
Given a $(\sm;\kappa;\xi)$-smoothing policy $\pi_{\vtheta}$, the spectral norm of the policy Hessian can be upper-bounded as follows:
\begin{equation*}
    \|\mathcal{H}\eta_\pi\|\leq \frac{c}{(1-\gamma)^2}\left(\frac{2\gamma\psi^2}{1-\gamma}+\kappa+\xi\right)+\frac{2R^2_{max}\psi^2}{(1-\gamma)^3},
\end{equation*}
where
\begin{alignat*}{2}
    c & \coloneqq \sup_{s\in\Sspace,a\in\Sspace}|\mathcal{R}(s,a)-\lambda(\mathcal{R}(s,a)-J_\pi)^2| \\
      &= \max\left\{\min\left\{\frac{1}{4\lambda}+J_\pi; R_{max}+4\lambda R^2_{max}\right\}; \pm\big[R_{max}-\lambda(R_{max}-J_\pi)^2\big] \right\}.
\end{alignat*}
\end{restatable}

\begin{proof}
First we note that:
\begin{equation*}
    \begin{aligned}
        |Q_\pi(s,a)-\lambda X_\pi(s,a)| &= \frac{1}{1-\gamma}\EV_{\substack{s'\sim d_{\pi}(\cdot|s)\\a'\sim\pi_{\vtheta}(\cdot|s')}}\big[ |R(s,a)-\lambda (R(s,a)-J_\pi)^2|\big] &\\
        &\leq \frac{c}{1-\gamma} &\forall s\in\Sspace,a\in\Aspace.
\end{aligned}
\end{equation*}
Then, using the same argument as in Lemma 6 from~\citep{papini2019smoothing}, the following upper bounds hold:
\begin{align*}
    &\| \nabla J_\pi \| \leq \frac{R_{max}\psi}{1-\gamma}, \\
    &\|\nabla (Q_\pi-\lambda X_\pi)\| \leq \frac{\gamma}{(1-\gamma)^2}c\psi.
\end{align*}
Finally, applying triangle and Jensen inequalities on Equation \ref{eq:Heta}:
\begin{align*}
    \|\mathcal{H}\eta_\pi\|\leq &\EV_{\substack{s\sim d_{\mu,\pi}\\a\sim\pi_{\vtheta}(\cdot|s')}}\big[\|\nabla\log\pi_{\vtheta}\nabla^\top(Q_\pi-\lambda X_\pi)\|\big]+ \EV_{\substack{s\sim d_{\mu,\pi}\\a\sim\pi_{\vtheta}(\cdot|s')}}\big[\|\nabla(Q_\pi-\lambda X_\pi)\nabla^\top\log\pi_{\vtheta}\|\big]\\
    &+\EV_{\substack{s\sim d_{\mu,\pi}\\a\sim\pi_{\vtheta}(\cdot|s')}}\big[\|\nabla\log\pi_{\vtheta}\nabla^\top\log\pi_{\vtheta}(Q_\pi-\lambda X_\pi)\|\big]\\
    &+\EV_{\substack{s\sim d_{\mu,\pi}\\a\sim\pi_{\vtheta}(\cdot|s')}}\big[\|\nabla\nabla^\top\log\pi_{\vtheta}(Q_\pi-\lambda X_\pi)\|\big]\\
    &+\frac{2}{1-\gamma}\|\nabla J_\pi\|^2.
\end{align*}
The application of the previous bounds and the smoothing assumption give the thesis.
\end{proof}

We can now see that Theorem~\ref{thm:safe_improvement} is just an adaptation of Theorem 9 from~\cite{papini2019smoothing} to the mean-volatility objective $\eta$, using the Hessian-norm bound from Lemma~\ref{lem:hba}.

In the case of stochastic gradient-ascent updates, as the ones employed in Algorithm~\ref{alg:pg}, this result can be directly used to derive optimal, safe meta-parameters for Algorithm~\ref{alg:pg}: 

\begin{restatable}{corollary}{safeparams}\label{cor:safeparams}
Let $\Pi_{\Theta}$ be a smoothing policy class, $\vtheta\in\Theta$ and $\delta \in (0,1)$. Given a $\delta$-confidence bound on the gradient estimation error, \ie an $\epsilon_\delta>0$ such that:
\begin{equation}\label{eq:grad_est_error}
\mathbb{P}\bigg(\|\widehat\nabla_N\eta_{\vtheta}-\nabla\eta_{\vtheta}\|\leq \frac{\epsilon_\delta}{\sqrt{N}}\bigg)\geq 1-\delta\qquad \forall \vtheta\in\Theta\, N\geq 1,
\end{equation}
the guaranteed performance improvement of the stochastic gradient-ascent update $\vtheta_{k+1} = \vtheta_k+\alpha\widehat{\nabla}_N \eta_{\vtheta_k}$ is maximized by step size $\alpha^* = \frac{1}{2L}$ and batch size $N^*= \bigg\lceil{\frac{4\epsilon_\delta^2}{\|\widehat\nabla_N\eta_{\vtheta_k}\|^2}}\bigg\rceil$. Moreover, with probability at least $1-\delta$, the following non-negative performance improvement is guaranteed:
\begin{equation*}
    \eta_{\vtheta_{k+1}} - \eta_{\vtheta_k} \geq \frac{\|\widehat{\nabla}_N \eta_{\vtheta_k}\|^2}{8L}.
\end{equation*}
\end{restatable}

Again, this is just an adaptation of Corollary 14 from~\cite{papini2019smoothing} to the mean-volatility objective $\eta$, using the Hessian-norm bound from Lemma~\ref{lem:hba}.

Under a Gaussianity assumption on $\nabla \eta_{\vtheta}$, which is reasonable for a sufficiently large batch size $N$, the error bound $\epsilon_{\delta}$ can be derived from an F-distribution ellipsoidal confidence region:

\begin{theorem}
Let $\widehat\nabla_N\eta_{\vtheta}$ the mean of $N$ independent samples drawn from $\nabla\eta_{\vtheta}\sim\mathcal{N}_m(\vmu,\Sigma)$. Then:
\begin{equation*}
    \mathbb{P}(\|\widehat{\nabla}_N\eta_{\vtheta}-\nabla\eta_{\vtheta}\|\leq\epsilon_\delta)\geq 1-\delta, 
\end{equation*}
where
\begin{equation*}
    \epsilon_\delta \leq \sqrt{\frac{Nm}{N-m}\|S\|F_{1-\delta}(m; n-m)},
\end{equation*}
 $N$ is the batch size, $m$ is the dimension of the parameters space $\Theta$, $F_{1-\delta}(a;b)$ is the $\delta$ quantile of a F-distribution with $a$ and $b$ degrees of freedom, and $\|S\|$ is spectral norm of the sample variance matrix S generated by the gradient samples.
\end{theorem}
\begin{proof}
We recall Corollary 5.3 and Theorem 5.9 in ~\citep{hardle2007applied}, adopting the same notation\footnote{This means that we will use $x$ to refer to $\nabla\eta_{\vtheta}$ and $x_i$ for its $i$-th estimation.}.
Let $\vx_1,\dots,\vx_N\sim \mathcal{N}_m(\vmu,\Sigma)$.
The sample mean and sample variance are defined as: 
\begin{equation*}
    \begin{aligned}
        \overline{\vx} &= \frac{1}{N}\sum_{i=1}^N\vx_i, \\
        S &= \frac{1}{N}\sum_{i=1}^N(\vx_i - \overline{\vx})(\vx_i - \overline{\vx})^T.
    \end{aligned}
\end{equation*}
Then:
\begin{equation*}
    (\overline{\vx}-\vmu)^\top S^{-1}(\overline{\vx}-\vmu)\sim\frac{1}{N-1} T^2(m; N-1) = \frac{m}{N-m}F(m;N-m),
\end{equation*}
where $T^2(a;b)$ and $F(a;b)$ are respectively the Hotelling's $T^2$ and $F$ distribution with $a$ and $b$ degrees of freedom.\\ 
Consequently, the following standard result provides a confidence region for $\vmu$:
\begin{proposition}
For all $\delta\in(0,1)$, $\mathbb{P}(\vmu\in E) \geq 1-\delta$, where $E$ is the following set:
\begin{equation*}
    E = \left\{\vx\in\mathbb{R}^m : (\overline{\vx}-\vx)^\top S^{-1}(\overline{\vx}-\vx) < \frac{m}{N-m}F_{1-\delta}(m;N-m)\right\},
\end{equation*}
and $F_{1-\delta}(a;b)$ is the $(1-\delta)$-quantile of the F-distribution with $a$ and $b$ degrees of freedom.
\end{proposition}
Hence, the estimation error $\vmu-\overline{\vx}$ is contained, with probability at least $1-\delta$, in the following set:
\begin{equation*}
 E_0 = \left\{\vx\in\mathbb{R}^m : \vx^\top S^{-1}\vx < \frac{m}{N-m}F_{1-\delta}(m; N-m)\right\},
\end{equation*}
which is bounded by the following ellipsoid:
\begin{equation*}
    \mathcal{E} = \left\{\vx\in\mathbb{R}^m : \vx^\top A\vx = 1\right\},
\end{equation*}
where $A=\left(\frac{mF_{1-\delta}(m;N-m)}{N-m}S\right)^{-1}$.
As a consequence, the euclidean norm of the estimation error is upper bounded by the largest semiaxis of the ellipsoid:
\begin{equation*}
    \zeta_\delta \coloneqq \norm{\vmu-\overline{\vx}} \leq \max_{i\in\left\{1,\dots,m\right\}} \{c_i\},
\end{equation*}
where $c_1,\dots,c_m$ are the semiaxes of $\mathcal{E}$. The semiaxes can be derived from the matrix $A$ using the following equalities:
\begin{align*}
    &\eig_i(A) = \frac{1}{c_i^2} &\text{for $i=1,\dots,m$},
\end{align*}
where $\eig_i(A)$ denotes the $i$-th eigenvalue of $A$ (the order does not matter). Thus, we can bound the estimation error norm as:
\begin{align*}
    \epsilon \leq \frac{1}{\sqrt{\min_{i\in\left\{1,\dots,m\right\}}\{\eig_i(A)\}}}.
\end{align*}
Finally, we can just compute the largest eigenvalue of $S$:
\begin{align*}
    \min_{i\in\left\{1,\dots,m\right\}}\{\eig_i(A)\} &= \min_{i\in\left\{1,\dots,m\right\}}\left\{\eig_i\left(\left(\frac{mF_{1-\delta}(m;N-m)}{N-m}S\right)^{-1}\right)\right\} \\
    &= \frac{N-m}{mF_{1-\delta}(m;N-m)}\max_{i\in\left\{1,\dots,m\right\}}\left\{\eig_i\left(S\right)\right\},
\end{align*}
Leading to:
\begin{align*}
    \zeta_\delta \leq \sqrt{\frac{m} {N-m}F_{1-\delta}(m;N-m)\norm{S}},
\end{align*}
with probability at least $1-\delta$, where $\norm{S}$ denotes the spectral norm of the sample variance matrix $S$ (equal to the largest eigenvalue since $S$ is positive semi-definite).\\
Equation~\ref{eq:grad_est_error} is verified by defining
\begin{equation*}
    \epsilon_\delta \coloneqq \sqrt{N}\zeta_\delta.
\end{equation*}
\end{proof}

\section{Exponential Utility applied on the reward}\label{app:exp_reward}
In this section we show that the exponential utility applied on the reward approximates the mean-volatility objective. \\
We consider the following measures:
\begin{align*}
    J_\pi &= (1-\gamma) \EV_{\tau|\pi}\big[\sum_t \gamma^t R_t\big]\\
    M_\pi &= (1-\gamma) \EV_{\tau|\pi}\big[\sum_t \gamma^t R^2_t\big]\\
    \nu_\pi^2 &= (1-\gamma)\EV_{\tau|\pi}\big[\sum_t (R_t -J_\pi)^2\big],
\end{align*}
where $R_t = R(S_t, A_t)$ is the reward obtained at time t.\\
Let's take into account now the exponential utility applied on the reward, and its second-order Taylor expansion:
\begin{equation*}
    U(R) = e^{-c R} = 1 -cR+ \frac{c^2}{2}R^2+ o(c^3)
\end{equation*}
Hence, if we sum all the discounted utilities of the rewards and take the expected value, we obtain:
\begin{align*}
    \EV_{\tau|\pi}\big[\sum_t \gamma^t e^{-c R_t}\big] &= \EV_{\tau|\pi}\big[\sum_t \gamma^t (1- cR_t +\frac{c^2}{2} R^2_t) + o(c^3)\big] \\
    &= \frac{1}{1-\gamma} - c \EV_{\tau|\pi}\big[\sum_t \gamma^t R_t] + \frac{c^2}{2} \EV_{\tau|\pi}\big[\sum_t \gamma^t R^2_t\big]+ o(c^3)   \\
&= \frac{1}{1-\gamma} -  \frac{c}{1-\gamma}J_\pi + \frac{c^2}{2(1-\gamma)}M_\pi + o(c^3)    
\end{align*}
Again, consider the Taylor expansion applied to the logarithm:
$$\log(\alpha + x) \approx \log(\alpha) + \frac{x}{\alpha}- \frac{x^2}{2\alpha^2}$$
As a consequence the following loose approximation holds:
\begin{align*}
    \log(\EV_{\tau|\pi}\big[\sum_t \gamma^t e^{-cR_t}\big]) &= - \log{(1-\gamma)} - cJ_\pi + \frac{c^2}{2}M_\pi - \frac{(1-\gamma)^2}{2}[(-\frac{cJ_\pi}{1-\gamma}+ \frac{c^2}{2(1-\gamma)}M_\pi)^2] + o(c^3) \\
    &\approx -\log(1-\gamma) - cJ_\pi + \frac{c^2}{2}(M_\pi-J_\pi^2)
\end{align*}
Consequently:
$$ \max_{\pi} -\frac{1}{c}\log(\EV_{\tau|\pi}[\sum_t \gamma^t e^{cR_t}])\approx \max_\pi J_\pi - \frac{c}{2}[M_\pi-J_\pi^2]$$ 

Finally, following the definition of $\nu_\pi^2$, we can obtain the first-order approximation.
\begin{align*}
    \nu_\pi^2 &= (1-\gamma) \EV_{\tau|\pi}[\sum_t \gamma^t (R_t - J_\pi)^2] \\
    &= (1-\gamma)  \EV_{\tau|\pi} [\sum_t \gamma^t (R^2_t + J_\pi^2 - 2 R_t J_\pi)]  = M_\pi- J_\pi^2
\end{align*}

\section{Experiments}\label{app:exp}
\subsection{Portfolio Optimization}

The task parameters we used are specified below:
{\begin{center}\begin{tabular}{ccc}
$T = 50$; & $r_l = 1.001$; & $N = 4$; \\
$r_{nl}^{high}=2$; & $r_{nl}^{low}=1.1$;&$M=10$; \\
$p_{risk} = 0.05$; & $p_{switch} = 0.1$; &$\alpha = \frac{0.2}{M}$. \\
\end{tabular}\end{center}}
There are $M+1$ possible actions, and the policy we used is a neural network with two hidden layers and 10 neuron per hidden layer.

\subsection{S\&P Trading Framework}
In the trading environment, actions $a_t \in \{-1,0,1\}$, where $-1,1$ indicate short and long positions, and $0$ stands for flat position.
The value of the asset at time $t$ is $p_t$, and the reward is equal to $$R_t = a_t(p_t-p_{t-1})-f|a_t-a_{t-1}|,$$
where $f$ indicates the fees, set to $7/100000$.\\
the policy we used is a neural network with two hidden layers and 64 neurons per hidden layer. 

\end{document}